\documentclass[12pt]{article}
\usepackage{natbib}
\usepackage[figuresright]{rotating}
\usepackage[T1]{fontenc}
\usepackage{verbatim}
\usepackage[monochrome]{color}
\usepackage{graphicx}
\usepackage{caption}
\usepackage{subcaption}
\usepackage{algorithm}
\usepackage{algpseudocode}
\usepackage{bm}
\usepackage{url}
\usepackage{multirow}
\usepackage{amssymb,amsmath,amsthm}
\usepackage{thmtools}
\usepackage[toc,page]{appendix}
\DeclareMathOperator*{\diag}{diag}

\DeclareMathOperator*{\argmin}{arg\,min}

\DeclareMathOperator{\ASE}{ASE}

\DeclareMathOperator*{\ISVT}{SVT}
\DeclareMathOperator*{\NMF}{NMF}

\DeclareMathOperator*{\MSE}{MSE}

\newtheorem{xcondition}{Condition}

\DeclareMathOperator{\pdiag}{diag}

\declaretheorem[style=definition]{example}

\newtheorem{xthm}{Theorem}[section]

\usepackage{hyperref}
\hypersetup{
    pdfpagemode=UseOutlines,    
}
\newcommand{\blind}{0}
\title{Techniques for clustering interaction data as a collection of graphs}
\def\spacingset#1{\renewcommand{\baselinestretch}%
{#1}\small\normalsize} \spacingset{1}
\spacingset{1.45}
\begin{document}
\if0\blind
{
	\title{\bf Techniques for clustering interaction data as a collection of graphs}
	\author{Nam Lee\thanks{nhlee@jhu.edu}, Carey Priebe, Youngser Park\hspace{.2cm}\\
		   Department of Applied Mathematics and Statistics\\
		    Johns Hopkins University \\
		    and \\
		   I-Jeng Wang\hspace{.2cm}\\
		   Applied Physics Lab \\ 
		   Johns Hopkins University \\
		    and \\
		   Michael Rosen \hspace{.2cm}\\
		   Armstrong Institute for Patient Safety and Quality \\
		   Johns Hopkins University}
	\maketitle
} \fi
\if1\blind
{
  \bigskip
  \bigskip
  \bigskip
  \begin{center}
    {\LARGE\bf Techniques for clustering interaction data as a collection of graphs}
\end{center}
  \medskip
} \fi
\bigskip
\begin{abstract}
A natural approach to analyze interaction data of form ``what-connects-to-what-when'' is to 
create a time-series (or rather a sequence) of graphs through 
temporal discretization (bandwidth selection) and spatial discretization
(vertex contraction).
Such discretization together with non-negative factorization techniques can be useful for obtaining clustering of graphs.  
Motivating application of performing clustering of graphs (as opposed to vertex clustering) can be 
found in neuroscience and in social network analysis, and it can also be used to 
enhance community detection (i.e., vertex clustering) by way of conditioning on the cluster labels. 
In this paper, we formulate a problem of clustering of graphs as a model selection problem. 
Our approach involves information criteria, non-negative matrix factorization and singular value thresholding, and we illustrate 
our techniques using real and simulated data.    
\end{abstract}
\noindent%
{\it Keywords:} High Dimensional Data, Model Selection, Network Analysis, Random Graphs
\section{Introduction}\label{sec:intro}
A typical data set collected from a network of actors is a collection 
of records of \emph{who-interacted-with-whom-at-what-time}, and  for 
network analysis, one often creates a sequence of graphs from such data.    
For a study of neuronal activities in a brain (c.f.~\cite{jarrell2012connectome}), the actors can be 
neurons.
For a study of contact patterns in a hospital in which potential disease transmission route is 
discovered (c.f.~\cite{10.1371/journal.pone.0073970}, \cite{GauvinPanissonCattuto}), the actors can be health-care professionals and patients in a hospital. 
In practice, transformation of the interaction data $\mathcal D$ to a time-series $\mathcal G$ of graphs  
uses temporal-aggregation and vertex-contraction, but there is no deep 
understanding of a proper way to perform such a transformation.  In this paper, 
we develop a model selection framework that can be applied to choose a transformation for interaction data, 
and develop a theory, on statistical efficiency of our model selection techniques in an asymptotic setting. 
In \cite{10.1371/journal.pone.0073970},
RFID wearable sensors were used to detect close-range interactions between individuals in 
a geriatric  unit of a hospital where health care workers and patients 
interact over a span of several days.   Then, for epidemiological analysis, it is examined whether or not
``the contact patterns were qualitatively similar from one day to the next''.
A key analysis objective there is identification of potential infection routes within the hospital .  
In this particular case, if there were two periods with distinct interaction patterns, then performing community detection 
on the unseparated graph can be inferior to performing on two separate graphs (See Example \ref{[exa:gclust-vclust]}).
In \cite{GauvinPanissonCattuto}, for a similar dataset describing the social interactions of students in a school, 
a \emph{tensor} factorization approach was used to detect the community structure, and to find an appropriate model to fit, 
the so-called ``core-consistency'' score from \cite{CEM:CEM801} was used.
For another example of such data set but in a larger scale, we utilize the data source called ``GDELT'' 
(Global Dataset of Events, Language, and Tone) introduced in \cite{leetaru2013gdelt},     
We follow the example below throughout this paper.  
\begin{example}\label{[exa:gdelt]}
GDELT  is continually updated by way of parsing news reports from various of news sources around the globe.  
The full GDELT data set contains
more than $200$ million entries of $(s,i,j,k)$-form spanning the periods from $1979$ to the present (roughly 12,900 days),
and the actors are attributed with $59$ features such as religions, organizations, location and etc. 
For more detailed description, we refer the reader to \cite{leetaru2013gdelt}.
The original data can be summarized in the following format:
\begin{align}
\mathcal D_T = \left\{ (s,i,j,k) : i,j \in V, s \in [0, T]\right\},
\end{align}
where $V=\{1,\ldots,n\}$ denotes `actors'
and $(s,i,j,k)$ denotes the event that `actor' $i$
perform type-$k$ action on `actor' $j$ at time $s$.  
In this paper, we will consider a subset of the data covering $48$ days of year $2014$. 
By aggregating the full data set \emph{by} day, and 
then by designating, say, $206$ actors according to their geo-political labels,  
we arrive at a time-series of graphs on $206$ vertices.  
More specifically, the particular discretization yields a sequence $\{G(t)\}$ of graphs, where each $G_{ij}(t)$ denotes 
the number of records with $(s,i,j,k)$, 
where $s$ belongs to $t$th day of $48$ days and $i$ and $j$ belong 
to one of $206$ geo-political labels.  
For each $t$, by applying Louvain algorithm (c.f.~\cite{1742-5468-2008-10-P10008}) for community detection to each $G(t)$, 
we can obtain a clustering $\mathcal C_{t}$ of $206$ vertices.  Then, for time $i$ and time $j$, we can compute the adjusted 
Rand index between $\mathcal C_{i}$ and $\mathcal C_{j}$.  When averaged across all pairs $i < j$, the mean value of the adjusted Rand index is slightly below $0.20$ (c.f.~Figure \ref{fig:1}). This suggests that for some pairs $(G(i),G(j))$ of graphs, 
the community structure of $G(i)$ and the community structure of $G(j)$ have a non-negligible overlapping feature.
A main question that we attempt to answer in this paper is whether or not a particular choice 
of discretization is efficient in some sense, i.e., 
to decide whether or not to further temporally aggregate $\{G(t)\}$ to \emph{a smaller collection of graphs}.  \hfill$\Box$
\end{example}
For another motivating example, 
consider the fact that interaction between $n$ neurons can be naturally modeled with graphs on $n$ vertices,  
where each edge weight is associated with the functional connectivity between neurons. 
Specifically, in \cite{jarrell2012connectome}, chemical and electrical neuronal pathways of C.~elegan 
worm were used to study the decision-making process of C.~elegan.  The area of studying a graph in such a way 
for further expanding our knowledge of biology is called ``Connectome''.  
While there is no ground truth answer because this is still a difficult science question,
we can still consider deciding whether or not combining two graphs into 
a single graph is more sensible with respect to a model selection principle.  
As a proxy, we follow in this paper a data example on Wikipedia hyperlinks, for which 
a more convincing but qualitative answer can be formulated for the same question.
\begin{example}\label{[exa:wiki]}
Wikipedia is an open-source Encyclopedia that is written 
by a large community of users (everyone who wants to, basically). 
There are versions in over 200 languages, with various amounts of content. 
Naturally, there are plenty of similarities between Wikipedia pages written in English (represented by 
an adjacency matrix $E$) and Wikipedia pages written in French (represented by $F$) since 
the connectivity between a pair of pages is driven by the relationship between 
topics on the pages.   Nevertheless, $E$ and $F$ are different since the pages in Wikipedia are 
grown ``organically'', i.e., there is no explicit coordination between English 
Wikipedia community and French Wikipedia community that try to enforce the 
similarity between $E$ and $F$.    Hence, 
the number of hyperlinks in $E$ and 
the number of hyper-links in $F$ 
 might be different owing to the fact that two graphs are
being updated/developed at a different rate. However, since both graphs are representation of the 
same underlying facts, there is a strong reason to believe that two Wikipedia graphs would have a ``nearly identical''
connectivity structure as the users continue to contribute.  Specifically, the adjusted Rand index value of 
the Louvain clustering of $E$ and the Louvain clustering of $F$ is slightly below $0.27$.  This suggests that 
the community structure of $E$ and the community structure of $F$ have a non-negligible overlapping feature.
\hfill$\Box$
\end{example}
One motivation behind selecting temporal discretization carefully rather than working with a single 
simply-aggregated graph is the potential benefit of conditioning-by-``graph label'' when performing 
community detection.  Community detection algorithms use the connectivity structure of a single 
graph for clustering vertices.  For multiple graphs, given that the cluster-labels are known, 
one can aggregate the graphs with the same label to a single graph with which one performs
community detection (c.f.~Example \ref{[exa:gclust-vclust]}).  
Instead of our approach in this paper, the tensor factorization approach from \cite{GauvinPanissonCattuto}
together the core-consistency score heuristic from \cite{CEM:CEM801} can also be used.   
However, the tensor factorization form, PARAFAC, considered in \cite{GauvinPanissonCattuto}
is not as flexible as the matrix factorization form that we consider in this paper, and 
the core-consistency score heuristic from \cite{CEM:CEM801} can be too subjective just as 
an elbow-finding strategy of the principle component analysis can be too subjective.  
Also, spatial aggregation, i.e., vertex contraction, arises naturally in many applications.
For example, when analysis of neuronal activities in a brain, a group of neurons are often 
identified as a single group as a function of their physical region in the brain.  There are many level of 
granularity that one can explore, but it is not clear which level of granularity is sufficient for statistically 
sound analysis.  
In this paper, we introduce model selection techniques that address these issues.  
To do this, the rest of this paper is organized as follows.  
In Section \ref{sec:backgmat}, we review some necessary  background materials.  
In Section \ref{sec:modeldescription}, we give a generative description of our model for multiple 
random graphs as a dynamic network. This gives a ground for formulating our 
model selection criterion later.    In Section \ref{sec:main.results}, we present our main contribution. 
Specifically, we present a model selection technique for clustering of graphs based on non-negative factorization,
singular value decomposition, and their relation to singular value thresholding. 
We also present a convergence criteria for non-negative factorization algorithms
based on a fixed point error formula, for comparing competing non-negative factorization algorithms.  
Throughout our discussion, we illustrate our approach with numerical 
experiments using real and simulated data.
\begin{figure}
\centering
\begin{subfigure}{0.45\textwidth}
\centering
\includegraphics[width=\textwidth]{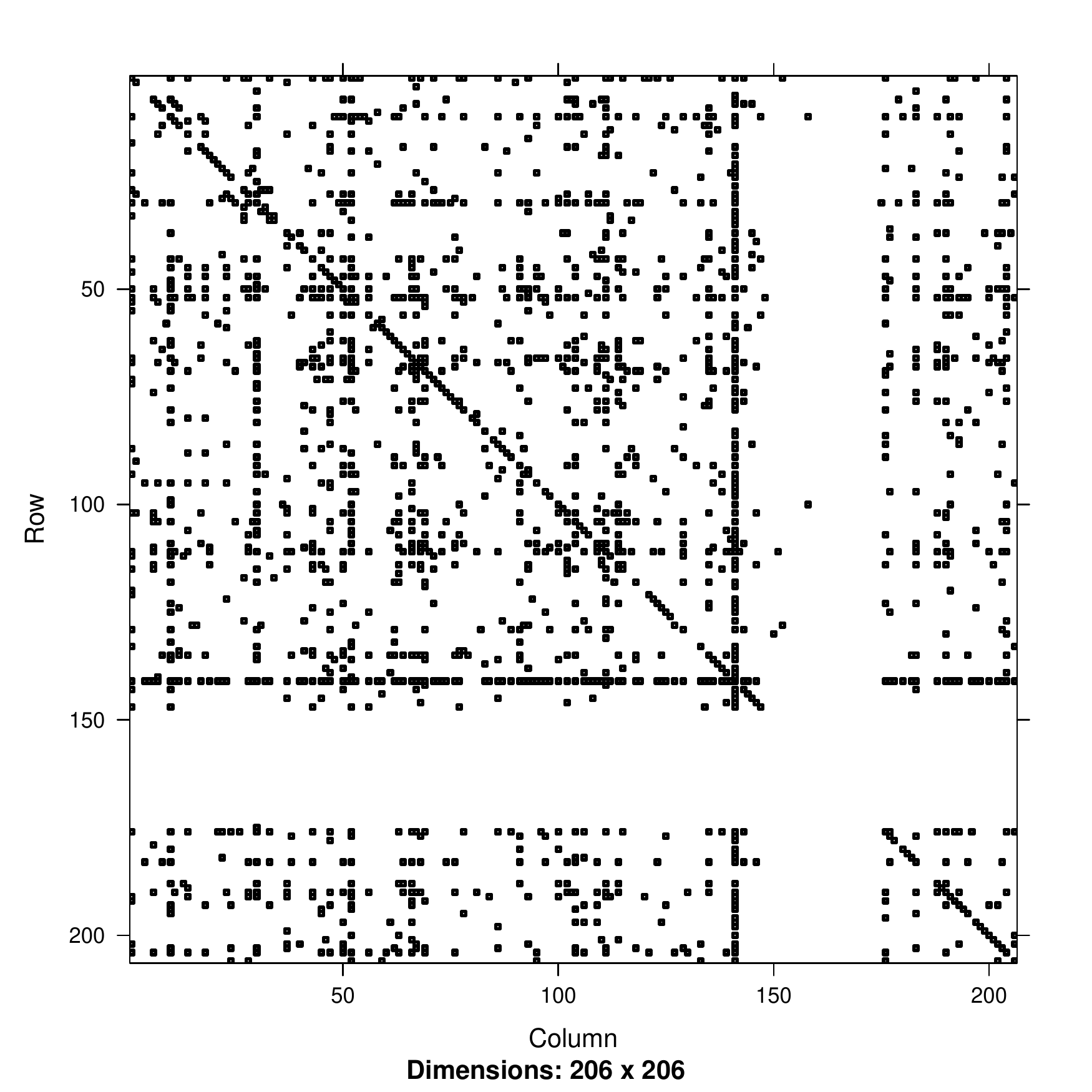}
\caption{$G(1)$}
\end{subfigure}
\qquad
\begin{subfigure}{0.45\textwidth}
\centering
\includegraphics[width=\textwidth]{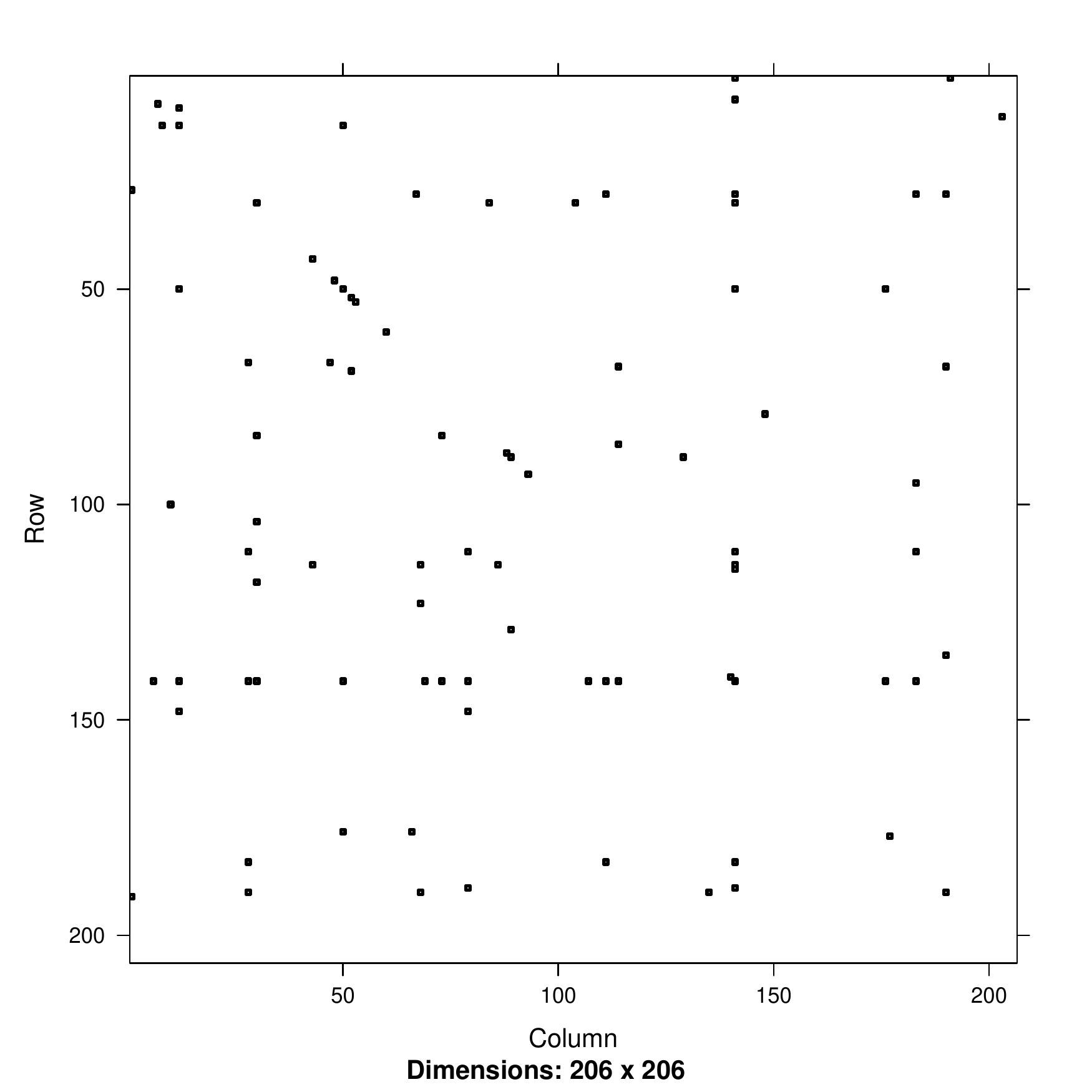}
\caption{$G(23)$}
\end{subfigure}\\
\begin{subfigure}{\textwidth}
\centering
\includegraphics[width=0.9\textwidth]{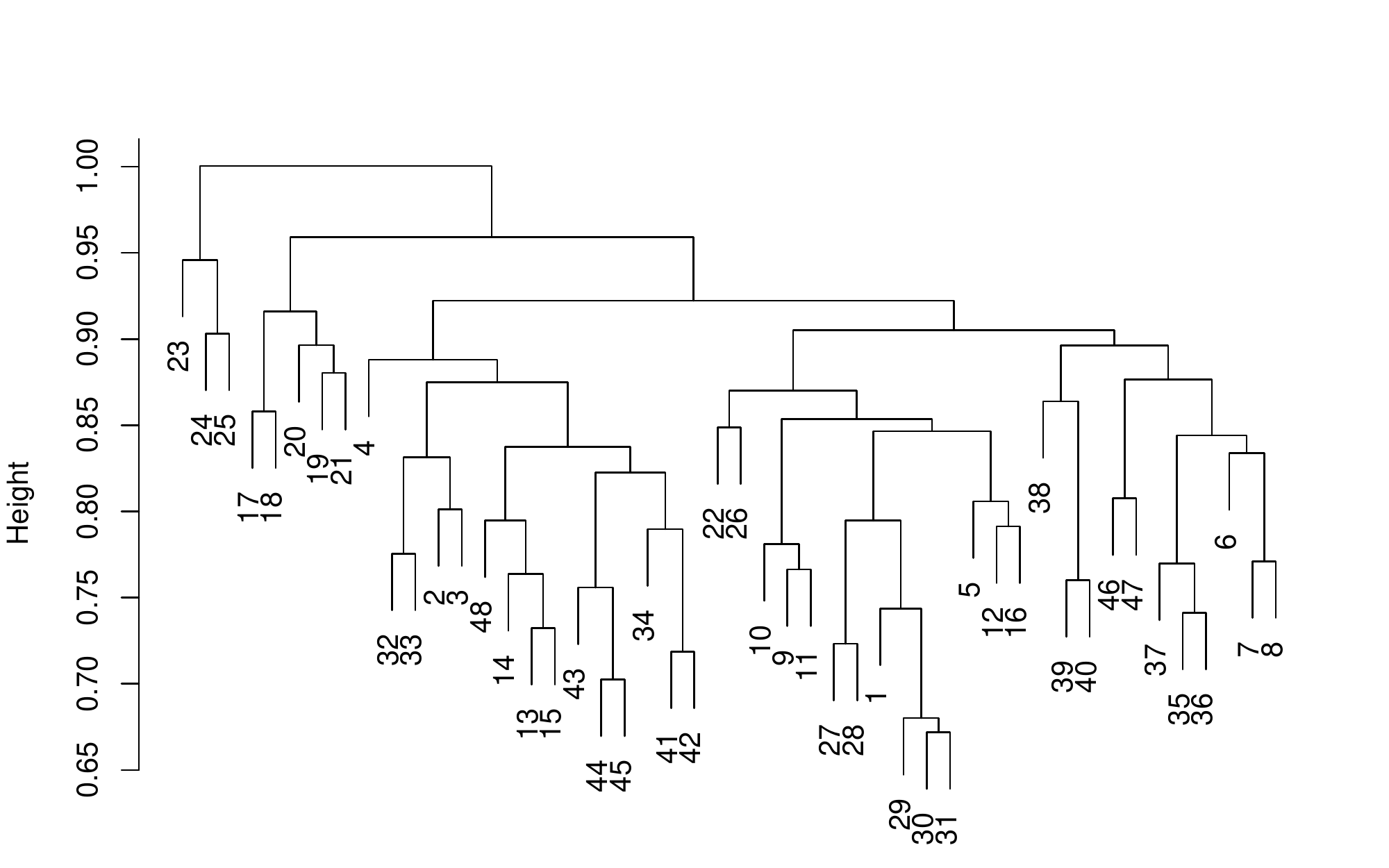}
\caption{Cluster Dendrogram}\label{fig1:c}
\end{subfigure}
\caption{In (a) and (b), two representative graphs from Example \ref{[exa:gdelt]} (GDELT data) are illustrated.  
In (c), a cluster dendrogram is constructed for visualization by using a dissimilarity matrix $D = (D_{ij})$, with  
$D_{ij} = \exp(-a_{ij})$, where for $i < j$, $a_{ij}$ denotes the adjusted Rand index between Louvain community detection clustering of $206$ vertices using 
$G(i)$ and Louvain community detection clustering of $46$ vertices using $G(j)$, and with $D_{ii} = 0$. When the tree were to be cut at height $0.96$, then $G(23), G(24), G(25)$ form one group
and the rest form the other. 
While it is clear that $G(1) \neq G(23)$, it can be argued that $G(23)$ is approximately a subgraph of $G(1)$.}\label{fig:1}
\end{figure}
\section{Background Materials}\label{sec:backgmat}
In this section, we briefly present necessary backgrounds, specifically, on random dot product graphs, 
adjacency spectral embedding, singular value thresholding and non-negative factorization.
First,  we review the dot product model for random graphs which can be seen to be a specific example of {\em latent position graphs} of 
\cite{Hoff2002}.   It can also be seen that the celebrated Erdos-Renyi random graph is an example of the 
random dot product model.           For any given $n \geq 1$, let ${Y}$
be a $n \times d$ matrix whose rows $\{Y_i\}_{i=1}^n$ are elements of $\mathbb R^{d}$.
The adjacency matrix $A$ of a random dot product graph
(RDPG) with {\em latent positions} ${Y}$ is
a random $n \times n$ symmetric non-negative 
matrix such that each of its entries takes a value in $\{0,1\}$ and 
$ \mathbf{P}[{A} | (Y_i)_{i=1}^{n}] =
  \prod_{i < j} (Y_i Y_j^\top)^{{A}_{ij}} (1 - Y_i
  Y_j^\top)^{1 - {A}_{ij}},$
where we assumed implicitly that each $Y_{i}$ takes a value in a subset $\mathcal S$ of $\mathbb{R}^{d}$ such that 
for each pair $\omega, \omega^\prime \in \mathcal S$, $0 \leq \langle \omega, \omega' \rangle \leq 1$. 
Given an $n\times d$ matrix of latent positions ${Y}$, the random dot product
model generates a symmetric (adjacency) matrix $A$ whose edges
$\{{A}_{ij}\}_{i <  j}$ are independent Bernoulli random
variables with parameters $\{{P}_{ij}\}_{i < j}$ where
${P} = YY^\top$.   As a slight generalization of this, we also consider a random dot product \emph{Poisson} graph, where we 
only require $YY^{\top}$ is a non-negative matrix and $A_{ij}$ is a Poisson random variable, and in this case, each $P_{ij}$ 
can take values in $[0,\infty)$ rather than $[0,1]$. 
Next, for an adjacency spectral embedding (ASE)
in $\mathbb R^d$ of (Bernoulli) graph $G$ (c.f.~\cite{STFP2012} and~\cite{AAetal-2013}), 
one begins by computing its singular value decomposition 
$USV^\top$ of $G$, where the singular values are placed in 
the diagonal of $S$ in a non-increasing order.  Note that given that $A$ is an $n\times n$ matrix,  
$1 \le d \le n$ and $d$ may or may not equal $n$.  
Then $ASE(G):=U_d S_d^{1/2}$, where 
$U_d$ is the first $d$ columns of $U$ and $S_d^{1/2}$ is an $d\times d$
diagonal matrix whose $k$th diagonal entry is the square root of the $k$th 
diagonal entry of $S$.  Provided that the generative model for 
$G$ is such that $\mathbf E[G|\xi] = \xi \xi^\top$,
under some mild assumptions, one can expect that clustering of 
the rows of $ASE(G)$ and clustering of the row vectors of $\xi$ 
coincide for the most part up to multiplication by an orthogonal matrix, and this is a useful fact when 
one perform clustering of vertices. A precise statement of the clustering error rate can be 
found in \cite{STFP2012,AAetal-2013}.
In \cite{STFP2012}, a particular choice for an estimate $\widehat d$ for $d$ was also motivated by way of 
an asymptotically almost-sure property of the singular values of a data matrix,
suggesting to take $\hat d$ to be the largest singular value of the matrix 
that is greater than $3^{1/4} n^{3/4} \log^{1/4}(n)$.
Next, the rank-$r$ singular value thresholding (SVT) of an $n\times m$ random matrix $M = U S V^{\top}$ 
is an estimate $\widehat M$ of $\mathbf E[M]$, and
$\widehat M = U_r S_r V_r^\top$, where $U_{r}$ and $V_r$ are the first $r$ 
columns of $U$ and $V$ respectively and $S_{r}$ is the first $r\times r$ upper sub matrix of $S$.
Specifically, if $M$ is symmetric, then $\widehat M = ASE(M)ASE(M)^\top$ when the embedding dimension 
$d$ of $ASE(M)$ coincides with the rank $r$. 
An error analysis expressed 
in terms of $\frac{1}{nm} \mathbf E[\|M-\widehat M\|_F^2]$
is given in \cite{CaiCandesShen2010} under some mild assumptions,
in an asymptotic setup in which $n\rightarrow\infty$. 
In \cite{Chatterjee2013}, a choice for $r$ was suggested yielding 
a so-called \emph{universal} singular value thresholding algorithm (USVT). 
In particular,
if it can be believed that there is no missing value, then, 
the USVT criterion chooses $r$ to be $\sqrt{(n\wedge m) \times 2.02}$, where $2.02$ can be replaced with 
any arbitrary number  greater than $2$. 
Lastly, implementation of our techniques will involve use of a non-negative factorization algorithm. 
For a non-negative matrix $M = L R$ for some full rank
non-negative matrices $L$ and $R$, 
finding the pair $(L, R)$ is known to be an NP-hard problem 
even if the value of $M$ is known exactly (c.f.~\cite{gaujoux2010flexible}). 
The number of columns of $L$ (and equivalently the number of rows of $R$) is said to be 
the \emph{inner dimension} of the factorization $M=LR$, and this terminology is regardless of 
the rank of $M$.    There are various algorithms 
for obtaining the factorization \emph{approximately} by numerically 
solving an optimization problem, e.g. among many choices,  one can take
\begin{align}
(\widehat L, \widehat R) := \argmin_{W\ge 0, H \ge 0} \| M - W H \|_F 
+ \alpha \|W\|_F + \beta \|H\|_F, \label{eqn:min-L2-error}
\end{align}
where $\alpha$ and $\beta$ are non-negative constants (c.f.~\cite{gaujoux2010flexible}).  
\section{Multiple graphs from a dynamic network}\label{sec:modeldescription}
We now introduce a generative model for multiple graphs, 
where each edge in a graph is generated/updated incrementally.  This allows us to model 
$\mathcal D$ as a data generated by multiple Poisson processes whose intensity 
functions are (potentially) inhomogeneous in time.  Rather than explicitly stating 
the form of the intensity functions, we allow our description of the discretized form
to implicitly specify the form of intensity functions.  
We consider a case that $n$ vertices generate $T$ graphs 
wherein $r$ repeated motifs are expressed.  To begin, we assume 
a partition of $[0,\tau] = [\tau_0, \tau_1) \cup \cdots \cup [\tau_{T-1},\tau_{T})$, 
where $\tau_0 = 0$, $\tau_{T} = T$ and $\tau_{i} < \tau_{i+1}$. 
Each event (from the underlying Poisson process) induces a record $(s,i,j)$, 
which should read ``at time $s$, 
interaction between vertex $i$ and vertex $j$ was needed''.
The $t$th graph is created by counting all update events occurred 
during interval $[\tau_{t-1},\tau_{t})$. For interval $[\tau_{t-1},\tau_{t})$, 
an update event occurs at a constant rate $\overline \Lambda_{tt}$, 
and then, each update event is attributed to the $k$th motif 
of $r$ (candidate) motifs with probability $\overline H_{k,t}$.  Subsequently, 
the update event is attributed to a particular vertex pair $(i,j)$ with 
probability $\overline W_{ij,k}$.   
Then, we arrive at a sequence of (potentially integrally-weighted) 
graphs $G(1),\ldots, G(T)$, where $G_{ij}(t)$ is the number of records of 
$(s,i,j)$ with $s \in [\tau_{t-1},\tau_{t})$. 
The data generated by such a network can be compactly written using 
an $n^2 \times T$ non-negative random matrix
$X$, where each $X_{\ell,t}$ represents the number of times that the $\ell$-th
ordered pair of vertices, say, vertex $i$ and vertex $j$, were needed for
an update event during interval $t=1,\ldots, T$.
In other words, $X$ is the matrix such that 
its $t$th column is a \emph{vectorized} version of $G(t)$, where 
the same indexing convention of vectorization is used for $G(1),\ldots, G(T)$.
Furthermore, by assumption, $\{X_{\ell,t}: \ell=1,\ldots, n^2, t=1,\ldots T\}$
are independent Poisson random variables such that for
some $n^2 \times r$ non-negative deterministic matrix $\overline W$,
$r \times T$ non-negative deterministic matrix $\overline H$, and $T\times T$ non-negative
deterministic diagonal matrix $\overline \Lambda$,
$
\mathbf E[X] = \overline W \overline H \overline \Lambda,
$
where we further suppose that
$\bm 1^\top \overline W = \bm 1^\top$ and $\bm 1^\top \overline H = \bm 1^\top$.
For each $t$, the total number $N(t)$ of events during time $t$ is
equal to $\bm 1^\top X e_t$, where $e_t$ denotes the standard basis vector
in $\mathbb R^T$ whose $t$th coordinate is $1$.
The random variables $N(1),\ldots, N(T)$ are then independent Poisson random variables,
and $\mathbf E[N(t)] = \overline \lambda_t := \overline \Lambda_{tt}$.
In general, $X$ is a noisy observation of $\overline X = \mathbf E[X]$.  As such, 
our problem of clustering of graphs requires finding the estimate $\widehat r$ of $r$, i.e.,
the number of repeated motifs and finding estimates $(\widehat W, \widehat H)$ of 
$(\overline W, \overline H)$ so that 
$
\overline X \approx \widehat W \widehat H \widehat \Lambda,
$
where $\approx$ reads ``is approximated with'' with respective to some loss criteria.  
When the number $n$ of vertices is sufficiently large, under certain simplifying
assumptions, a class of procedures known as singular value thresholding
(c.f.~\cite{CaiCandesShen2010} and \cite{Chatterjee2013}) 
can be used to effectively remove noise from random graphs in an $L_2$ sense, 
provided that each entry of $X$ is a bounded random variable.
We now conclude this section with the following three observations. 
\emph{First}, the values of $N(1),\ldots, N(T)$ are \emph{not} 
integral to our clustering of graphs.  Rather, we take $N(t)$ as the number  
of samples obtained for $t$th period, and for clustering, 
the object that we should focus is $\overline H$.  
\emph{Second}, while for our clustering 
of graphs, the columns of $\overline H$ subsume standard basis vectors, 
our description does not require to be such.  
However, for a model identification issue as well as 
efficiency of numerical algorithms for non-negative factorization, the restriction 
that the columns of $\overline H$ contains the full standard basis is,
while not necessary, critical (c.f.~\cite{uniqueNMF-HuangSS14}).  
\emph{Lastly}, the additivity property of Poisson random variables, i.e., the sum of independent  
Poisson random variable again being a Poisson random variable, greatly simplify our 
analysis involving temporal aggregation, and vertex-contraction, i.e., the operation which 
collapses a group of vertices to a single (super) vertex, aggregating their 
edge weights accordingly. 
\section{Main results}\label{sec:main.results}
\subsection{Overview}
In Section \ref{sec:SVT}, we introduce singular value thresholding 
as a key step in choosing vertex contraction. 
In Section \ref{sec:AIC}, we introduce a model selection
criteria for choosing the number of graph-clusters, under an asymptotic setting 
where the number of parameters grows.  In Section \ref{sec:FPE}, we introduce a numerical 
convergence criterion for non-negative factorization algorithm which quantifies the quality 
of $\widehat W$ and $\widehat H$ individually in addition to $\|X - \widehat W \widehat H\|_{F}$. 
Our  discussion assumes the following simplifying condition. 
\begin{xcondition}\label{cond:rank=innerdim}
A non-negative matrix $\overline X = \overline W \overline H \overline \Lambda$ is a rank $r$ matrix, and there exists
a unique non-negative factorization $\overline W \overline H$ with inner dimension $r$. \hfill$\Box$
\end{xcondition}
\subsection{On Denoising Performance of Singular Value Thresholding}\label{sec:SVT}
A temporal discretization policy determines the number $T$ of graphs, and a spatial discretization policy 
determines the number $n$ of vertices.  Subsequently, the number of parameters to estimate using data $\mathcal D$ then 
grows with the value of $\max\{n,T\}$.  As such, it is of interest to derive $X$ from $\mathcal D$ so that 
$\mathbf E[X]$ can be estimated from $X$ with a reasonable performance guarantee.  
One way to control the number of vertices in a graph is to perform vertex contraction. 
To be more specific, let $G$ be a graph on $n$ vertices, and then, let  $A := J  G J^\top$,
where $J$ is a partition matrix of dimension $m \times  n$.   That is, $\bm 1^{\top} J = \bm 1^{\top}$ 
and each entry of $J$ is either $0$ or $1$.  Essentially, the matrix $J$ acts on $G$ by aggregating 
a group of vertices to a single ``super'' vertex.  For simplicity, we assume that $(n/m)$ is the number of vertices in $G$ being contracted to a
vertex in $A=JGJ^\top$, whence $(n/m)^2$ is the number of entries in $G$ 
being summed to yield a value of an entry in $A$.  
Then, $A$ is a (weighted) graph on 
$m$ vertices. Next, let $\Delta$ be an $m\times m$ matrix such that 
for each vertex $u$ and vertex $v$,  $\Delta_{uv} := \frac{A_{uv} - \mathbf E A_{uv}}{\sqrt{\mathbf E A_{uv}}}$.
Finally, we may ``sketch'' $\Delta$ so as to further 
reduce the data to a smaller $p\times p$ matrix $\delta$.  More specifically, 
we take $\delta = S \Delta S^\top$, where $S$ is a $p\times m$ full rank matrix
such that each row is a standard basis in $\mathbb R^m$. 
We call $\delta$
the residual matrix, and when $\mathbf E A_{uv}$ is replaced with an estimate 
$\widehat A_{uv}$, we write $\widehat \delta$ and call an \emph{empirical} 
residual matrix for $\widehat A_{uv}$.   
For clustering of vertices to be meaningful, it is preferable to keep the value of $p$ large but 
to keep the total number of parameters to estimate in check, it is preferable to keep $p$ small enough.  
Keeping this in mind, to choose the vertex contraction matrix $J$, we propose to use the singular values of the empirical 
residual matrices $\{\widehat \delta(t)\}_{t=1}^T$,
where $\widehat \delta(t)$ is the empirical residual matrix for $G(t)$.  
Specifically, given $\widehat \delta = \widehat \delta(t)$ is a $p\times p$ matrix, 
we propose o compute its singular values  $\{\widehat \sigma_{\ell}\}$ of $\widehat \delta$, and 
then also to compute the expected singular values $\{\overline \sigma_{\ell}\}$ of a random matrix 
having the same dimension as $\widehat \delta$ whose entries are i.i.d.~standard normal random variables. 
Then, we propose to use 
$\MSE(\widehat\delta):=\sqrt{\frac{1}{p}\sum_{\ell=1}^{p} |\widehat \sigma_{\ell} - \overline \sigma_{\ell}|^{2}}$ 
to quantify the quality of the vertex contraction matrix $J$.
In Theorem \ref{thm:sparse2dense},
we identify an asymptotic configuration for a tuple $(p,m,n)$ under which a null distribution of $\delta$ is derived.
\begin{xthm}\label{thm:sparse2dense}
Let $(p_{n},m_{n},n)$ be such that  $p_{n} < m_{n} <  n$, $n/m_{n} \rightarrow\infty$ and $p_{n}^{2} m_{n}/n \rightarrow 0$ as $n \rightarrow\infty$. 
Suppose that $G$ be a random dot product \emph{Poisson} graph.
Suppose that exists $\gamma > 0$ such that
for all sufficiently large $n$, $\min_{uv} \mathbf E[A_{uv}]/(n/m)^{2} \ge \gamma$ and that $\max_{ij} \mathbf E[G_{ij}] <\infty $.
Then, as $n\rightarrow\infty$, the sequence of $\delta$ converges 
to a matrix of independent standard normal random variables.  
\end{xthm}
For an estimate $\widehat A$ of $\mathbf E[A]$ to be used in $\widehat \delta$, we propose to perform
singular value thresholding from \cite{Chatterjee2013}. 
While direct application of their theorems to
our present setting is not theoretically satisfactory as Poisson random variables have unbounded support,
an asymptotic result can be obtained. 
To state this in a form that we consider, take an $n^{2}\times T$ random matrix $X$ whose entries $(X_{ij,t})$ 
are independent Poisson random variables.  Specifically, each column of $X$ is a vectorization of graph on $n$ vertices. 
Given a constant $C>0$, for each $ij$ and $t$, let
$
Y_{ij,t} :=  X_{ij,t} \wedge C := \min\{X_{ij,t}, C\}.
$
Then, we let $\widehat Y$ be the result of the singular value threholding of $Y$ taking 
$\widehat r$ to be the number of positive singular value greater than $\sqrt{2.02 \min\{T,n^{2}\}}$.
Under various simplifying assumptions that the upper bound $C$ does not grow too fast with respect to the value of $T$ and $n$, it can be shown by adapting the proofs in  \cite{Chatterjee2013} that
\begin{align}
\lim_{T \wedge n \rightarrow\infty }\MSE(\widehat Y;X) = 0, \label{thm:meantotalinteraction}
\end{align}
where  
\begin{align*}
\MSE(\widehat Y;X) := \mathbf E\left[ \frac{1}{n^2T} \|\widehat Y - \mathbf E[X]\|_F^2\right].
\end{align*}
An appealing feature of this singular value thresholding procedure is that  in comparison to, say, 
a maximum likelihood approach, its computational complexity is relatively low  when $n^{2}$ and $T$ grow.  
Also, it can be post-processed with a maximum likelihood procedure if computational cost is not prohibitive. 
On the other hand, our discussion thus far on $\widehat Y$ 
relied on using the \emph{universal} singular value thresholding, and 
in next section, we touch  on the issue of refining this universal choice to a particular one. 
\subsection{Criteria based on Asymptotic Analysis of a Penalized Loss Method}\label{sec:AIC}
To motivate our discussion in this section, we begin with the following example 
which illustrates a reason why clustering of graphs might be relevant to performing community detection. 
\begin{example}\label{[exa:gclust-vclust]}
Let $\Pi$ be a permutation matrix corresponding to permutation $(264)(1)(3)(5)$.
Then, let $\widetilde M = L R$ and $\overline M = L \Pi (\Pi^{\top}L^{\top})$ where for $\kappa = 0.1$, 
\begin{align}
R := L^{\top} :=
\frac{1}{1+\kappa}
\begin{pmatrix}
\kappa & 1 & 1 & \kappa & 0 & 0 \\
1 & \kappa & 0 & 0 & \kappa & 1 \\
0 & 0 & \kappa & 1 & 1 & \kappa
\end{pmatrix}.
\end{align}
Treating $\widetilde M$ and $\overline M$ as adjacency matrix of weighted graphs,  
for $\widetilde M$, the (intended) vertex-clustering consists of $\{1,6\},\{2,3\},\{4,5\}$, but 
for $\overline M$, the (intended) vertex-clustering consists of 
$\{1,2\}, \{3,4\}, \{5,6\}$.  Now, let $\overline X$ be an $36 \times 10$ matrix such that each column of $
\overline X$ is the vectorization of either $\overline M$ or $\widetilde M$ and there are five from $\overline M$ 
and five from $\widetilde M$.  
In our numerical experiment, application of a community detection algorithm 
(Louvain)  applied to  $\overline M$ and $\widetilde M$ separately produced a correct clustering of vertices.
On the other hand,  when the aggregation is performed across the columns of $\overline X$ regardless of their graph-labels, 
the same community detection algorithm yields  $\{1,2\}, \{3,4\}, \{5,6\}$, hiding the clustering structure of $\widetilde M$. 
Our algorithm (see Algorithm \ref{algo:mainalgorithm} and  \ref{algo:mainalgorithm2} in Appendix 
\ref{sec:algorithms} for a sketch of the steps) finds the correct inner dimension of $\overline X$ and also finds the
correct clustering of the column of $\overline X$. 
\hfill$\Box$
\end{example}
\begin{table}
\centering
\caption{AICc values for Example \ref{[exa:gclust-vclust]}.  Each column of matrix $\overline X$ is
the vectorization of a non-negative matrix, say, $A_{t}$, which can either be $\widetilde M$ or $\overline M$.  Because performing vertex clustering on the aggregated matrix $\sum_{t} A_{t} = 5 \widetilde M + 5 \overline M$ can be inferior to performing vertex clustering on $\widetilde M$ and $\overline M$ separately, it is appealing to correctly identify the columns of $\overline X$ as a vectorization of either $\overline M$ or $\widetilde M$.  Application of an NMF procedure can produce the correct labels provided that we know the fact that the number of ``active'' patterns are two.}
\begin{tabular}{cccc}
 $\widehat r$ & Loss &  Penalty &      AICc \\ \hline
   1&      28.31480 & 0.0750000 &28.38980\\
   2&      24.93259 & 0.3000000 &{\bf 25.23259}\\
   3&      24.93251 & 0.7625668 &25.69507\\
   4&      24.91407 & 1.5360164 &26.45009\\ \hline
\end{tabular}
\end{table}
Our overall approach is a penalized maximum likelihood estimation.
In particular, 
our derivation of the penalty term in \eqref{eqn:penterm-aicc} is akin to the one in \cite{davies2006estimation}, in which 
for a linear regression problem, the penalty term is derived 
by computing the bias in the Kullback-Leibler discrepancy (c.f.~\cite[pg. 243]{linhart1986model}).   
To make it clear, we denote by $r^{*}$ the true inner dimension of $\overline X = \overline W \overline H \overline \Lambda$ factorization. 
We introduce an information criterion AICc as a part of our clustering-of-graphs technique.
Specifically, we choose $r$ by finding the minimizer of the mapping $r\rightarrow \text{AICc}(r)$. 
Our model-based information criterion 
(AICc) is obtained by appropriately penalizing the log-likelihood of the Poisson based 
model.   Specifically, we define the optimal choice for the number $r^*$ of cluster to be 
the smallest positive integer $r$ that minimizes
\begin{align}
\text{AICc}(r) & := - \sum_{ij,t} (\widehat W^{(r)} \widehat H^{(r)})_{ij,t} \log((\widehat W^{(r)} \widehat H^{(r)})_{ij,t}) \label{eqn:loglikterm-aicc}\\
& + \frac{1}{2} \sum_{k=1}^r (\widehat C_k^{(r)}-1)/\widehat Q_k^{(r)}, \label{eqn:penterm-aicc}
\end{align}
where $(\widehat W^{(r)}, \widehat H^{(r)})$ is such that
$\|M - \widehat W^{(r)} \widehat H^{(r)}\|_F^2
= \inf_{(W,H)} \|M - W H\|_F^2 + \alpha \|W\|_F + \beta \|H\|_F$
with $(W,H)$ ranging over ones such that
$\bm 1^\top W = \bm 1^\top \in \mathbb R^r$ and $\bm 1^\top H = \bm 1^\top \in \mathbb R^T$,
$\widehat Q_k^{(r)} = \sum_{t=1}^T N_t \widehat H_{kt}^{(r)}$,
$\widehat C_k^{(r)} = \sum_{ij} \bm 1\{\widehat W^{(r)}_{ij,k} > 0\}$.
Intuitively, as $r$ increases, 
the term in \eqref{eqn:loglikterm-aicc} is expected to 
decreases as the model space becomes larger, but 
the term in \eqref{eqn:penterm-aicc} is expected to increase for a larger value of $r > r^{*}$ especially when $\widehat W_{ij,k}^{(r)} >0$ and 
$\widehat W_{ij,k^\prime}^{(r)} > 0$ for $k\neq k^\prime$ for many values of $ij$ (in other words, when some columns of $\overline W$ are ``overly'' similar to each other, the penalty term becomes more prominent).
To begin our analysis, we consider a sequence of problems, where each problem is indexed 
by $\ell$ so that for example, we have a sequence of collections of 
$\mathcal G^{(\ell)} = \{G^{(\ell)}(t)\}_{t=1}^{T}$.  
\begin{xcondition}\label{condition:poisson}
Suppose that for each $t$, almost surely,
\begin{align}
\lim_{\ell\rightarrow\infty} N_t^{(\ell)}/\ell = \overline \lambda_t. 
\end{align}
\end{xcondition}
The dependence of $\mathcal G^{(\ell)}$ on $\ell$ is only through 
Condition \ref{condition:poisson}. 
Note that $\overline W$ and $\overline H$ do not depend on $\ell$
even under Condition \ref{condition:poisson}.
To simplify our notation,  we suppress the dependence of our notation  
on $\ell$   unless  it is necessary. 
Also, with slight abuse of notation, for each $k$, we write 
$\overline \lambda_k$ for the value of $\overline \lambda_t$ for the case that time-$t$ class label $k(t) = k$. 
Also, we let $\overline n_k = |\{t: \kappa(t) = k\}|$.
Let
$\varphi (W,H) := -\sum_{t}  \frac{1}{N_t} \mathbf E\left[ \sum_{ij} X_{ij,t} \log( (WH)_{ij,t})\left| \bm N \right. \right]$,
where \emph{$X$ is distributed according to one specified by the parameter $(\overline W, \overline H, \bm N)$}, and 
$W$ and $H$ are dummy variables. 
Note that each $\mathbf E[\varphi_{t}(\widehat W,\widehat H)\left| \bm N\right. ]$ is 
the expected overall KL discrepancy for the $t$th column of $X$, where 
$\varphi_{t}(W,H) := - \sum_{ij}  X_{ij,t} \log( (WH)_{ij,t})$.
Our next result shows that the connection between our AICc formula and $\varphi(\widehat W,\widehat H)$.
\begin{xthm}\label{thm:result/main}
Under Condition \ref{condition:poisson}, almost surely,
\begin{align}
\lim_{\ell\rightarrow\infty} \ell \left( 
\mathbf E[\varphi(\widehat W, \widehat H)\left| \bm N \right. ] - \varphi(\overline W, \overline H) \right) = 
\frac{1}{2} \sum_{k=1}^{r} \frac{\overline Z_k -1}{\overline n_k \overline \lambda_k},
\end{align}
where $\overline Z_k = \sum_{ij} \bm 1\{\overline W_{ij,k} > 0\}$, provided that 
$ \widehat W \widehat H  = X \diag(\bm 1^{\top} X)^{-1}$.
\end{xthm}
For each $t$, $Xe_t$ is a complete and sufficient statistic for $N_{t}$ independent multinomial trials 
whose success probability is specified by $\overline W e_{\kappa(t)}$.  Similarly and trivially, the data matrix $X$
constitutes  a complete and sufficient statistic for $\sum_{t} N_{t}$ trials whose success probability is 
$\overline W \overline H$. Then, our AICc is a function of $X \pdiag(1/N_{1},\ldots, 1/N_{T})$, more specifically, 
a function of a complete and sufficient statistic.
Then, by Lehman-Scheffe,
if the expected value of AICc were \emph{identical} to the expected \emph{weighted}  overall KL discrepancy, then 
our AICc would be an \emph{uniformly minimum variance unbiased estimator} (UMVUE) 
of the expected \emph{weighted}  overall KL discrepancy.   This motivates our formula for AICc.  
Next, we illustrate using AICc for model selection using real data examples. 
\begin{table}
\centering
\caption{AICc values for Example \ref{[exa:gdelt:aic]}.  
The optimal AICc value suggests that there are two clusters, where times $t=23, 24, 25$  are 
to be aggregated and times $t=1,\ldots, 22, 26, \ldots, 48$  are to be aggregated
}\label{tab:gdelt-aggregation}
\begin{tabular}{cccc}
$\widehat r$ & Loss &  Penalty &      AICc \\ \hline
     1  &      346.2536 & 0.01233593 & 346.2659 \\
      2  &      342.6821 & 0.16133074 & {\bf 342.8434} \\
      3  &      342.5578 & 0.59156650 & 343.1493 \\ 
     4    &  342.7041 &1.22363057 &343.9277 \\ \hline
\end{tabular}
\end{table}
\begin{example}[Continued from Example \ref{[exa:gdelt]}]\label{[exa:gdelt:aic]}
Our discussion here reiterates our result reported in Table \ref{tab:gdelt-aggregation}. 
Our clustering-of-graphs procedure (Algorithm \ref{algo:mainalgorithm} and Algorithm \ref{algo:mainalgorithm2}) performed 
on $\{G(t)\}_{t=1}^{48}$ picks $\widehat r = 2$ as the best model inner dimension,  
yielding cluster $\{G(t)\}_{t=23,24,25}$ and cluster $\{G(t)\}_{t\neq23,24,25}$.  This result corresponds to 
cutting the tree in Figure \ref{fig1:c} at height $0.96$.  Performing the clustering procedure again on 
 graph $A = \sum_{t=23,24,25} G(t)$ and  graph $A^{\prime} =\sum_{t\neq 23,24,25} G(t)$ 
yields that  the AICc value of $23.01160$ for $\widehat r = 1$ and  the AICc value 
of $23.16797$ for $\widehat r = 2$.  In words, this can be attributed to the facts that 
(i) $A^{\prime}$ is nearly a subgraph of $A$, and that (ii) $A^{\prime}$ is sparse, i.e., relatively small number of non-zero entries. 
This can be used to suggest performing clustering of vertices on the aggregated graph 
$A^{*} = \sum_{t=1}^{48} G(t)$.
\hfill$\Box$
\end{example}
\begin{example}[Continued from Example \ref{[exa:wiki]}] 
We apply our approach to decide whether or not two graphs are from the same ``template''.
We take the data matrix $X$ to be a matrix such that the first column of $X$ is the vectorization of $E$ 
(English Wikipedia  graph), and the second column of $X$ is the vectorization of $F$ (French Wikipedia graph).
Then, we can decide whether the inner dimension of $X$ is $1$ or $2$ using the AICc criterion. 
Our computation yields the AICc value of $22.05825$ for $r=1$ and the AICc value of $23.20715$ for $r=2$.
Therefore, our analysis suggests that both graphs have the same connectivity structure. 
\hfill$\Box$
\end{example}
\subsection{A Fixed Point Error Convergence Criterion for an NMF algorithm}\label{sec:FPE}
Our formulation of the AICc in \eqref{eqn:loglikterm-aicc} need not depend on a particular choice 
of non-negative factorization algorithm.  In particular, in \eqref{eqn:loglikterm-aicc}, we stated our 
formula using a modified ``Lee-Seung'' algorithm that minimizes $L_{2}$ error with $L_{1}$ regularizers
(c.f.~\cite{gaujoux2010flexible}).
On the other hand, there are many other options that can take its place, namely, ``Brunet'' algorithm (c.f.~
\cite{gaujoux2010flexible}). Then, one can ask if one is better than the other in some sense. 
In this section, we provide a way to compare these competing choices. 
For a non-negative matrix $\overline X$, which need not be symmetric, and with an approximate factorization $\overline X \approx W H$ with 
its inner dimension $r$, we write  $\varepsilon(W;\overline X) := \|\bm F(W,H)\|_F$ and $\varepsilon(H; \overline X) := \|\bm G(W,H)\|_F$,
where $\overline X = \overline U \overline \Sigma \overline V^\top$ is a singular value decomposition 
of $\overline X$, and 
\begin{align}
& \bm F(W,H) := W -  \overline X H^\top W^\top (\overline U \overline \Sigma^{-2} \overline U^\top) W, \label{eqn:FWH}\\
& \bm G(W,H) := H^\top - \overline X^\top W H (\overline V \overline \Sigma^{-2} \overline V^\top) H^\top.\label{eqn:GWH}
\end{align}
When $\overline X=LR$ is an exact NMF of $\overline X$, given that the rank of $\overline X=LR$ is also the rank of the non-negative matrix $L$,
it can be shown that the pair $(L,R)$ is the only solution to the fixed point equation $(\bm F(W,H), \bm G(W,H)) = \bm 0$.   
\begin{example}[Continued from Example \ref{[exa:gclust-vclust]}]\label{[exa:cont]}
It can be shown that 
for each $\kappa \in [0,0.5)$, the non-negative matrix $\overline X :=\overline W \overline H$ is uniquely 
non-negative matrix factorizable and for $\kappa \in [0.5,1]$,
that $\overline X = \overline W \overline H$
is not uniquely non-negative matrix factorizable (c.f.~\cite{uniqueNMF-HuangSS14}).
In Figure \ref{fig:Example-Part3}, for $\kappa =0.1$, we compared 
two non-negative factorization algorithm using our fixed point error formula, and found that ``Brunet'' algorithm is more 
appealing than 
the modified ``Lee-Seung'' algorithm in terms of a convergence characteristic of $\varepsilon(H;\overline X)$.   
\hfill$\Box$
\end{example}
\begin{figure}
\centering
\begin{subfigure}[b]{0.48\textwidth}
\centering
\includegraphics[width=\textwidth]{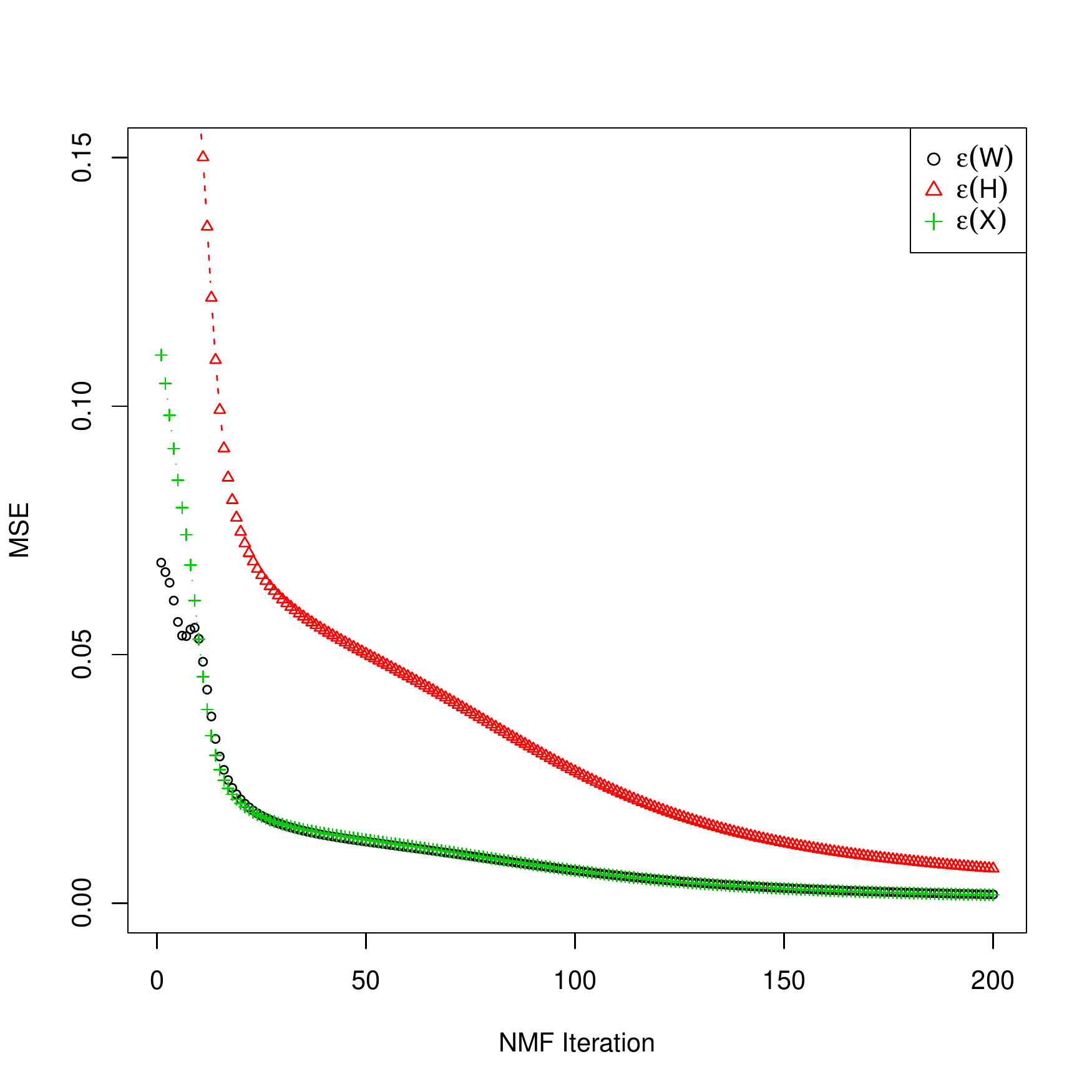}
\caption{Lee-Seung ($L_{2}$)}
\end{subfigure}
\begin{subfigure}[b]{0.48\textwidth}
\centering
\includegraphics[width=\textwidth]{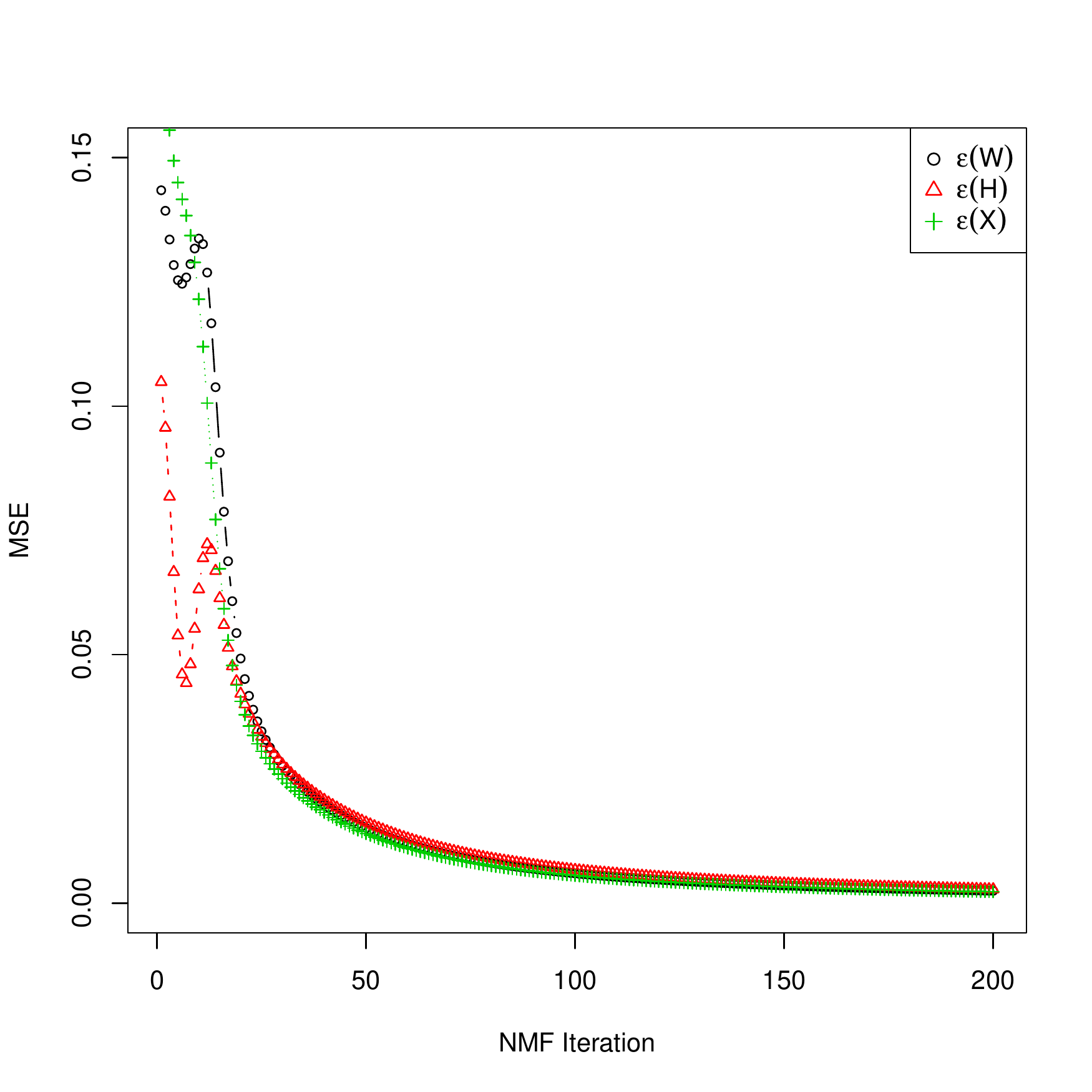}
\caption{Brunet (KL)}  
\end{subfigure}
\caption{Comparison of two NMF algorithms. For $\varepsilon(H;\overline X)$, ``Brunet'' outperforms ``Lee-Seung'' 
for this particular instance.  For illustration, 
the $1/6$ of the value of $\varepsilon(X;\overline X)$ is used instead.
}\label{fig:Example-Part3}
\end{figure}
We now discuss robustness of the fixed point error criteria for NMF.
We do this by exploring a connection between using of a non-negative factorization algorithm for clustering of vertices
(c.f.,~\cite{uniqueNMF-HuangSS14} and~\cite{gaujoux2010flexible}), 
and using adjacency spectral embedding for clustering of vertices (c.f.~\cite{AAetal-2013}).   
To begin, let $A$ be an $n\times n$ random matrix such that $\overline P=\mathbf E[A|Y] = Y Y^\top$, where
the rows $\{Y_i\}_{i=1}^n$ of $Y$ form a sequence of independent 
and identically distributed random probability vectors, i.e., $Y \bm 1 = \bm 1$,
and conditioning on $Y$, each $A_{ij}$ is an independent Bernoulli random variable.  
We write $\widehat Y$ for $\ASE(A)$. 
Non-negative factorization connects to the random dot product model by 
a simple observation that even when $n\times r$ matrix $Y$ is not non-negative, if $\overline P = Y Y^\top$ is a non-negative matrix with rank $r$, then
there exists an $n\times r$ non-negative matrix $W$ such that $\overline P = W W^\top$.   
Our next condition in Condition \ref{cond:ASEmeetsNMF} is a stronger version of this observation, and we assume
Condition \ref{cond:ASEmeetsNMF} to simplify our proof in  Theorem \ref{thm:randmat}.   Also, recall that $\widehat Y = \ASE(A)$ using the rank $r$.
\begin{xcondition}\label{cond:ASEmeetsNMF}
Suppose that for each $n$, 
there exists an orthogonal matrix $Q$ such that 
\begin{align}
\widehat Y Q =  \widehat W  := \argmin_{W \in \mathbb R_+^{n\times r}} \sum_{i<j} \left| \widehat P_{ij} - e_i^\top W W^\top e_j\right|^2,
\label{eqn:sym-nmf}
\end{align}
where $\widehat P = \widehat Y \widehat Y^{\top}$.
\end{xcondition}
  Now, given an estimate $\widehat W$ of $W$, it is often of interest to quantify how close $\widehat W$ is to $W$.  
While in practice, if $\|\overline P - \widehat W \widehat W^{\top}\|_{F} \approx 0$, then we expect 
$\widehat W$ to be close to $W$, but there is no way to know how close $\widehat W$ is to $W$.  
Our fixed point error formula $\varepsilon_n(\widehat W; \overline P)$ 
addresses this issue.  
The following technical condition is a key assumption in \cite{AAetal-2013}.
Our analysis relies on the main result in \cite{AAetal-2013}.    
\begin{xcondition}\label{cond:RowVectorDistinctEigenValue}\label{cond:ConditionNumber}
The distribution of $Y_i$ does not change with $n$, and 
the $r\times r$ second moment matrix $\Delta := \mathbf E[Y_i^\top Y_i]$
has distinct and strictly positive eigenvalues. 
Moreover, there exists a constant $\xi_0 < \infty$ such that almost surely,
for all $n$, 
\begin{align}
\overline \Sigma_{11}/ \overline \Sigma_{rr} \le \xi_0.
\end{align}
\end{xcondition}
\begin{xthm}\label{thm:randmat}
Under Condition \ref{cond:RowVectorDistinctEigenValue} and  \ref{cond:ASEmeetsNMF},
almost surely, 
\begin{align}
\limsup_{n\rightarrow\infty} \frac{ \varepsilon_n(\widehat W; \overline P)}{\sqrt{\log(n)}} < \infty.
\label{eqn:oracle-rate}
\end{align}
\end{xthm}
To put Theorem \ref{thm:randmat} into a perspective, we note that 
the largest value that  $(\varepsilon_{n}(\widehat W;\overline P))^{2}$ can take is 
$\sum_{i=1}^{n} \sum_{k=1}^{r} 1 = n r$ and $n/{\log(n)} \rightarrow\infty$ as $n\rightarrow\infty$. 
Specifically, almost surely, the 
``mean-square'' error will converges to zero, i.e., $
\limsup_{n\rightarrow\infty} \frac{1}{\sqrt n} \varepsilon_n(\widehat W;\overline P)  =0$. We also note that the 
result in \eqref{eqn:oracle-rate} is of a ``oracle'' type because  the value of $\overline P$ is unknown in practice. 
\section{Numerical Results}\label{sec:num.expr}
\label{sec:num.expr:sim}
We now examine performance of our AICc criteria using simulated data.  
We specify the general set-up for our Monte Carlo experiments. 
To begin, let  
\begin{align*}
\overline B^{(1)} :=
\begin{pmatrix}
0.1  & 0.045 & 0.015 & 0.19 & 0.001\\
0.045 & 0.05 & 0.035 & 0.14 & 0.03\\
0.015 & 0.035 & 0.08 & 0.105 & 0.04\\
0.19 & 0.14 & 0.105 & 0.29 & 0.13\\
0.001    & 0.03 & 0.04 & 0.13 & 0.09
\end{pmatrix}.
\end{align*}
Then, we set $\overline B^{(2)}$ to be the matrix obtained from 
$\overline B^{(1)}$ by permuting the rows by the permutation $(4152)$
and then by permuting the columns by the permutation $(43)$.  
Our specific choice for $\overline B_{uv}^{(1)}$ is motivated by 
the experiment data from \cite{Izhikevich04032008} in which ``Connectome'' is 
constructed to answer a biological question.
We consider random graphs on $n$ vertices such that  each $\mathbf E[G(t)]$
has a block-structured pattern, i.e., a checker-board like pattern.  For each $t=1,\ldots, T$, 
we take $G(t)$ to be a (weighted) graph on $n$ vertices, where each $G_{ij}(t)$ is a Poisson random variable. 
To parameterize the block structure, we set $n = 5 \times m$, and   
let $\kappa(t)$ be a deterministic label taking values in $\{1,2\}$. 
Then, we take $\mathbf E[G_{ij}(t)] = B_{uv}^{(\kappa(t))}$ for some $u$ and $v$.  
Specifically, the $(i,j)$th entry of $\mathbf E[G(t)]$ is taken to be $B_{v}^{(\kappa(t))}$ if
$i=5(u-1)+p$ and $j=5(v-1)+q$ for some $p,q=1,\ldots, m$.    
Our problem is then to estimate the number $r$ of clusters 
using data $G(1),\ldots, G(T)$, and the correct value for $\widehat r$ is $r=2$. 
We keep ${\overline\Lambda}_{11}  = \ldots = {\overline \Lambda}_{TT} > 0$,
so that there should not be any statistically-significant evidence in the total number of edges 
in the graph that will distinguish one cluster from another.  
For comparison,
we specify two other algorithms against which we compare our model selection procedure
(AICc o nmf), where o denote composition of two algorithms.   Our choices for two competing methods 
are based on an observation that our analysis of model selection procedure heavily relies on the fact that 
the rank and the inner dimension are the same. In our case, both the rank and the inner dimension of $\overline X$ 
is $2$, one way to estimate the value of $r$ is to use any algorithm for finding the number of non-zero singular values.
We denote our first baseline algorithm with (pamk o dist) and the second with (mclust o pca).  
These competing algorithms are often used in practice for choosing the rank of a (random) matrix. 
For (pamk o dist), we first compute  
the distance/dissimilarity matrix using pair-wise Euclidean/Frobenius distances between graphs,
and perform  \emph{partition around medoids} for clustering (c.f.~\cite{dudaHart1973}).
For (mclust o pca),
we first compute the singular values of the data matrix $X$ and use 
an ``elbow-finding'' algorithm to determine the rank of the data matrix (c.f.~\cite{ZhuGhodsi2006})
\begin{figure}
\centering
\includegraphics[height=3.5in]{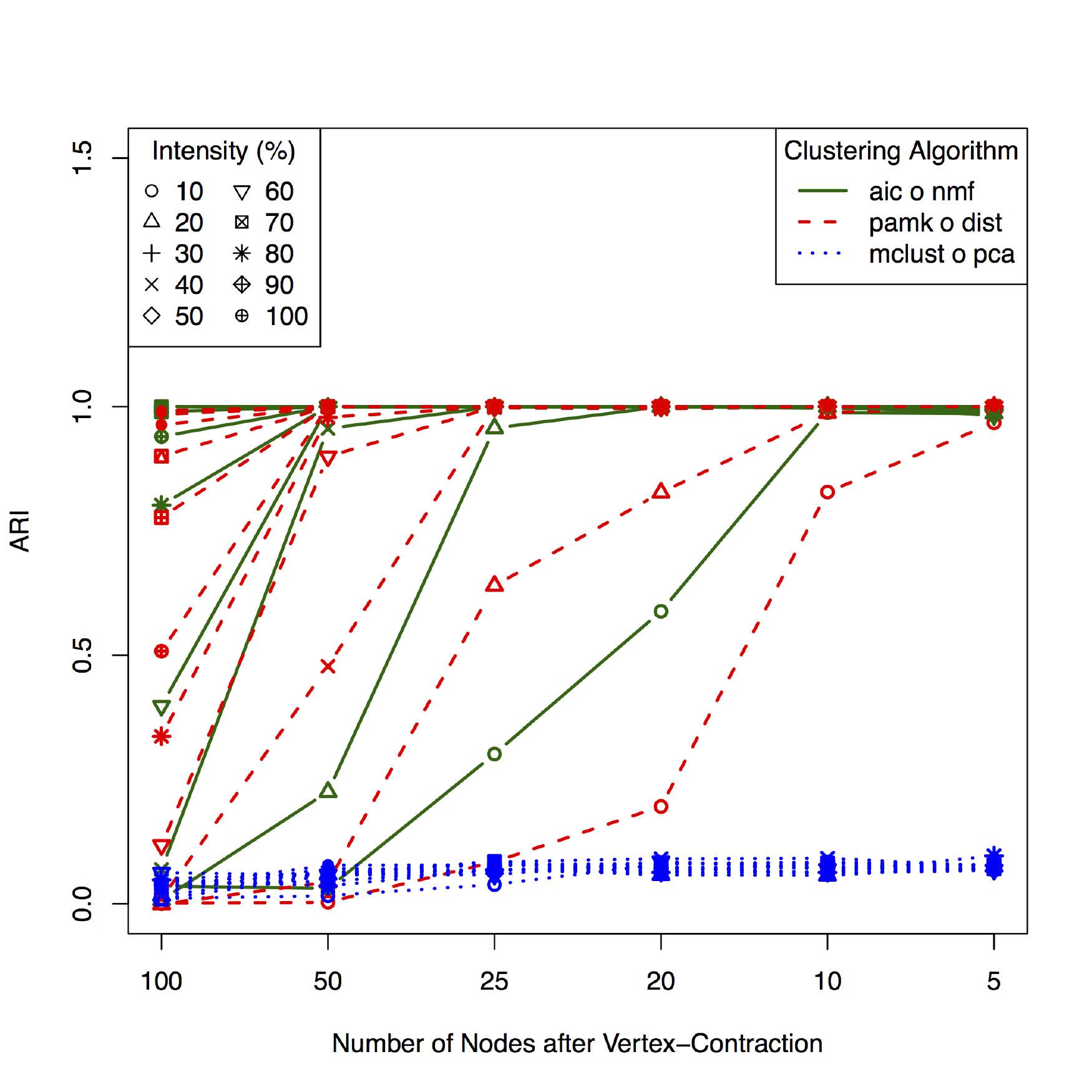}
\caption{Comparison of three  approaches through ARI for  the model selection performance.  
In all cases, our procedure either outperforms or nearly on par with the two baseline algorithms. 
}\label{fig:ARIasNodeSparsity}
\end{figure}
The result of our  experiment is summarized in Figure \ref{fig:ARIasNodeSparsity}.
In all cases, our procedure either outperforms or nearly on par with the two baseline algorithms. 
There are two parameters that we varied, the level of intensity and the level of aggregation. 
Parameterizing the level of intensity, for $\rho \in (0,1)$, we take 
$\overline \Lambda_t^{(\rho)} = \rho \overline \Lambda_t$, where a bigger value for $\rho$ means more chance 
for each entry of $X$ taking a large integer value. 
For the level of aggregation (or equivalently, vertex-contraction), 
if the number of nodes after vertex-contraction is $5$, the original graph is 
reduced to a graph with $5$ vertices.  Aggregation of edge weights is only done 
within the same block.  
Then, as the performance index, we use 
the adjusted Rand index (ARI) values (c.f.~\cite{WMRand1971}).  In general, ARI takes a value in $[-1,1]$. 
The cases in which the value of ARI is close to $1$ is ideal, indicating that clustering is consistent with the truth, and 
the cases in which the value of ARI is less than $0$ are the cases in which its performance is worse than randomly assigned clusters. 
\section{Conclusion}
In this paper,  we have considered a clustering problem that arises
when a collection of records of interaction is transformed into a time-series of graphs through discretization.   
Taking multiple graphs generated with any discretization scheme as a 
starting point, our work has addressed a question of whether or not the particular 
choice of discretization produces an efficient clustering of the multiple graphs.  In order to 
quantify the efficiency, we introduced a model selection criteria as a way to  
choose the number of clusters (c.f.~Theorem \ref{thm:result/main}).  
For choosing an appropriate non-negative factorization algorithm,  we have studied a fixed point formula 
as a convergence criteria for (numerical) non-negative factorization (c.f.~Theorem \ref{thm:randmat}).   
Throughout our discussion, our techniques are illustrated using various datasets.  In particular, along with 
Theorem \ref{thm:sparse2dense}, we have demonstrated,
through numerical experiments (c.f.~Section \ref{sec:num.expr}),
that choosing an appropriate vertex contraction can improve performance of our model selection techniques 
(c.f.~Algorithms \ref{algo:mainalgorithm} and \ref{algo:mainalgorithm2} outlined in Appendix F). 
The problem of choosing an appropriate vertex contraction  is still an open problem, and we consider this 
in our future work.
\subsection*{Supplementary Materials}
A collection of R codes and data that are used in this paper is stored in the following (temporary-for-review) web location:
\begin{center}
\url{https://www.dropbox.com/sh/r3q38lh5v5r6oel/AAAcmAZksvuczURVLi96ZYbVa?dl=0}
\end{center}
\if0\blind
{
\subsection*{Acknowledgment}
{This work is partially supported by Johns Hopkins University
Armstrong Institute for Patient Safety and Quality and
the XDATA program of the Defense Advanced Research
Projects Agency (DARPA) administered through
Air Force Research Laboratory contract FA8750-12-2-0303.  We also 
like to thank Runze Tang for useful comments throughout various stages of 
drafting.}
} \fi
\appendix
\section{Proof of Theorem \ref{thm:sparse2dense}}
\begin{proof}
To simplify our notation, we suppress the dependence of our notation on $n$. 
For example, we write $p, m$ for $p_{n},m_{n}$ respectively.
We denote by $\Phi$ the usual (cumulative) distribution function of a normal random variable and 
denote by $F_{uv}$  the cumulative distribution of $\Delta_{uv}$.
Note that each $A_{uv}$ is again a Poisson random variable.  
To be concrete, for each $u,v$, we let $\{\xi_{uv}(\ell)\}_{\ell=1}^{\infty}$ be a sequence of independent and identically distributed 
Poisson random variables such that $A_{uv} = \sum_{\ell=1}^{\tau_{uv}} \xi_{uv}(\ell)$,
where $\tau_{uv} := (n/m)^{2}$.
By assumption, $ (n/m)^{2} \rightarrow\infty$  as $n\rightarrow\infty$.
Now, by central limit theorem, we see that 
$\Delta_{uv}$ converges in distribution to a standard normal random variable.  
In fact, by Berry-Essen inequality 
\begin{align}
\max_{u,v=1,\ldots,m} \sup_{s\in \mathbb R} |\Phi(t) - F_{uv}(t)| 
\le & \max_{u,v=1,\ldots,m} \frac{C}{\sqrt{\mathbf E[A_{uv}}]}  \max_{\ell=1}^{\tau_{uv}} \frac{\mathbf E[(\dot \xi_{uv}(\ell))^{3}]}{\mathbf E[ \xi_{uv}(\ell)]} \\
\le & \max_{u,v=1,\ldots,m} \frac{C}{\sqrt{\tau_{uv} \mathbf E[A_{uv}]/\tau_{uv}}}  \max_{\ell=1}^{\tau_{uv}}  \frac{\mathbf E[(\dot \xi_{uv}(\ell))^{3}]}{\mathbf E[ \xi_{uv}(\ell)]} \\
\le & \max_{u,v=1,\ldots,m} \frac{C}{\sqrt{\tau_{uv}}\sqrt{ \gamma }}   \beta
\end{align}
where $\dot \xi_{uv}(\ell) :=\xi_{uv}(\ell)- \mathbf E[\xi_{uv}(\ell)]$,
$\gamma := \min_{uv} \mathbf E[A_{uv}]/\tau_{uv} $, 
$\beta := \max_{\ell,uv} \frac{\mathbf E[(\dot \xi_{uv}(\ell))^{3}]}{\mathbf E[ \xi_{uv}(\ell)]}$.
By assumption, we have that $0 < \gamma <\infty$ and $\beta < \infty$, 
and hence, there exists $C^\prime < \infty$ such that 
\begin{align}
\max_{u,v=1,\ldots,m} \sup \left|\Phi(s) - F_{uv}(s) \right| 
\le 
\frac{C^\prime}{ n/m },
\end{align}
and $n/m $ diverges as $n\rightarrow \infty$.  
Now, since $p^{2}(m/n)\rightarrow 0$, then the rate of convergence of $\delta_{uv}$ to a standard normal 
is uniform in index $uv$ in the following sense:
\begin{align*}
\lim_{n\rightarrow\infty}\sum_{u,v=1}^{p}\max_{u,v=1,\ldots,m} \sup \left|\Phi(s) - F_{uv}(s) \right| 
\le 
{C^\prime}  \lim_{n\rightarrow\infty} p^{2 }(m/n)= 0.
\end{align*}
This completes our proof. 
\end{proof}
\section{Proof of Theorem \ref{thm:result/main}}
\begin{proof}
For simplicity, we write $\overline \theta := \overline W \overline H$ and  
for any $\widehat W$ and $\widehat H$, we write $\widehat \theta = \widehat W \widehat H =  X \diag(\bm 1^{\top} X)^{-1}$. 
We have 
\begin{align*}
\varphi(\widehat W,\widehat H) 
= - \sum_{ij,t} \mathbf E[X_{ij,t}\left| \bm N \right. ]/N_t \log( (\widehat W \widehat H)_{ij,t})
= - \sum_{ij,t} \overline \theta_{ij,t} \log( \widehat \theta_{ij,t}).
\end{align*}
First, by way of a Taylor expansion of the $\log$ function, we note
\begin{align}
\mathbf E[\varphi(\widehat \theta)  \left| \bm N\right. ] 
= \varphi(\overline \theta) 
& - \mathbf E\left[ \sum_{ij,t} \overline\theta_{ij,t} \bm
1\{\overline\theta_{ij,t} > 0\} \frac{1}{\overline\theta_{ij,t}} (\widehat\theta_{ij,t} -
\overline\theta_{ij,t})\left| \bm N \right.  \right]\label{eqn:biascorrection-1} \\
& - \mathbf E\left[ \sum_{ij,t}
\overline\theta_{ij,t} \bm 1\{\overline\theta_{ij,t} > 0\} \frac{-1}{2\overline\theta_{ij,t}^2}
(\widehat\theta_{ij,t} - \overline\theta_{ij,t})^2 \left| \bm N \right.  \right] \label{eqn:biascorrection-1.5}\\
& -  \mathbf E[R(\widehat \theta_{ij,t}, \overline \theta_{ij, t}) \left| \bm N\right. ].\label{eqn:biascorrection-2}
\end{align}   
We will come back to the term $R(\widehat \theta_{ij,t}, \overline \theta_{ij, t})$, and we focus on the first two terms first.  
Since $\widehat \theta$ is an unbiased estimator of $\overline \theta$, we see that the first
term on the right in \eqref{eqn:biascorrection-1} vanishes to zero.  For the
 term in \eqref{eqn:biascorrection-1.5}, we note that since each $X_{ij,t}$
is a binomial random variable for  $N_t$ trials with its success probability
$\overline \theta_{ij,t}$, we see that
\begin{align}
    &\quad -\sum_{ij,t} \bm 1\{\overline\theta_{ij,t} > 0\} \frac{-1}{2\overline\theta_{ij,t}} \mathbf E[(\widehat\theta_{ij,t} - \overline\theta_{ij,t})^2\left| \bm N \right. ] \\
    & = \sum_{ij,t} \bm 1\{\overline\theta_{ij,t} > 0\} \frac{1}{2\overline\theta_{ij,t}} \frac{1}{N_t} \overline\theta_{ij,t} (1- \overline\theta_{ij,t}) \\
    & = \sum_{ij,t} \bm 1\{\overline\theta_{ij,t} > 0\} \frac{1}{2N_t} (1- \overline\theta_{ij,t}) \\
    & = \sum_{t=1}^T \frac{1}{2N_t} \left(\sum_{ij} \bm 1\{\overline\theta_{ij,t} >0\}\right) - \sum_{t=1}^T \frac{1}{2N_t} \left(\sum_{ij} \bm 1\{\overline\theta_{ij,t} > 0\} \overline\theta_{ij,t}\right) \\
    & =  \sum_{t=1}^T \frac{\overline Z_{\kappa(t)}}{2N_t} - \sum_{t=1}^T \frac{1}{2N_t},\label{eqn:biascorrection-3}
\end{align}
where the last equality is due to the fact 
that each column of $\overline \theta$ sums to one.  Hence, in summary, we 
see that 
\begin{align}
\lim_{\ell\rightarrow\infty} 
\ell (\mathbf E[\varphi(\widehat \theta)\left| \bm N \right. ] - \varphi(\overline \theta))
= \lim_{\ell\rightarrow\infty} \ell\sum_{t=1}^T \frac{\overline Z_t-1}{2N_t}. \label{eqn:AICc/proof/part/one}
\end{align}
Next, we note that in general,
\begin{align*}
\sum_{t=1}^T \frac{\overline Z_t}{N_t}
= \sum_{k=1}^r\sum_{t\in k} \frac{\overline Z_t}{N_t} 
= \sum_{k=1}^r {\overline Z_{t_k}} \sum_{t\in k}\frac{1}{N_t} 
= \sum_{k=1}^r 
\left(\frac{ \overline Z_{t_k} }{ \sum_{t \in k} N_{t} } 
\left(\sum_{t\in k}\frac{1}{N_t/\sum_{s \in k} N_{s}}\right)\right),
\end{align*}
where we write $t\in k$ for $\overline H_{kt} = 1$ for simplicity. Then,
\begin{align}
 \lim_{\ell\rightarrow\infty} \frac{1}{1/\sum_{s \in k} N_{s}} \sum_{t\in k}\frac{1}{N_t} 
=\lim_{\ell\rightarrow\infty} \sum_{t\in k}\frac{1}{(N_t/\ell)/\sum_{s \in k} (N_{s}/\ell)}
= \sum_{t\in k}\frac{1}{\overline\lambda_t/\sum_{s \in k} \overline \lambda_{s}} = 1,
\end{align}
where the last equality is due to the fact that for each $k$, $\{G(t):t\in k\}$ are identically distributed. 
Since $\lim_{\ell\rightarrow\infty} \sum_{t \in k} N_{t}/\ell
= \overline n_k \overline \lambda_k$, 
\begin{align}
\lim_{\ell\rightarrow\infty} \ell \sum_{t=1}^T \frac{\overline Z_t}{N_t}
=
\lim_{\ell\rightarrow\infty} \ell \sum_{k=1}^r \frac{ \overline Z_{t_k} }{ \sum_{t \in k} N_{t} } 
= 
  \sum_{k=1}^r \frac{ \overline Z_{t_k} }{ \lim_{\ell\rightarrow\infty} \sum_{t \in k} N_{t}/\ell }
=  \sum_{k=1}^r \frac{ \overline Z_{t_k} }{\overline n_k\overline \lambda_k},
    \label{eqn:AICc/proof/part/two}
\end{align}
where $t_{k}$ is any fixed $t \in k$.
We now turn to the remainder term $\mathbf E[R(\widehat \theta, \overline \theta)\left|\bm N\right.]$.
Specifically, 
\begin{align}
R(\widehat \theta, \overline \theta)
 = \sum_{k=3}^{\infty} \frac{1}{k}\frac{(-1)^{k+1}}{\overline \theta_{ij,t}^{k}}|\widehat \theta_{ij,t} - \overline \theta_{ij,t}|^{k}
 = \frac{-1}{\overline \theta_{ij,t}^{3}} (\widehat \theta_{ij,t} - \overline \theta_{ij,t})^{3} 
\sum_{k=0}^{\infty} \frac{(-1)^{k+1}}{\overline \theta_{ij,t}^{k}} \frac{(\widehat \theta_{ij,t} - \overline \theta_{ij,t})^{k}}{k+3}.
\end{align}
Hence, 
\begin{align}
\ell |R(\widehat \theta, \overline \theta)| 
\le  \frac{\ell}{\overline \theta_{ij,t}^{3}} |\widehat \theta_{ij,t} - \overline \theta_{ij,t}|^{3} 
\sum_{k=0}^{\infty} \frac{1}{\overline \theta_{ij,t}^{k}} \frac{|\widehat \theta_{ij,t} - \overline \theta_{ij,t}|^{k}}{k}.
\end{align}
Since $\widehat \theta_{ij,t} \rightarrow \overline \theta_{ij,t}$ almost surely,    
it can be shown that there exists a constant $c > 0$ such that for each sufficiently small $\varepsilon > 0$, 
for sufficiently large $\ell$,
with $1-\varepsilon$ probability,
$$
\sum_{k=0}^{\infty} \frac{1}{\overline \theta_{ij,t}^{k}} \frac{|\widehat \theta_{ij,t} - \overline \theta_{ij,t}|^{k}}{k} \le c.
$$
Moreover, using the third moment formula for a binomial random variable explicitly, we have 
\begin{align}
&\qquad \lim_{\ell\rightarrow\infty}\ell \mathbf E[ |\widehat \theta_{ij,t} - \overline \theta_{ij,t}|^{3}\left| \bm N \right.  ] \\
&\le 
\lim_{\ell\rightarrow\infty} \ell \frac{1}{N_{t}^{3}}  N_{t}
\overline \theta_{ij,t} (1-\overline \theta_{ij,t}) (1- 2 \overline \theta_{ij,t}) \\
&\le 
\lim_{\ell\rightarrow\infty} \frac{1}{N_{t}/\ell}
\overline \theta_{ij,t} (1-\overline \theta_{ij,t}) 
\lim_{\ell\rightarrow\infty}   \frac{1- 2 \overline \theta_{ij,t}}{N_{t}} = 0.
\end{align} 
In summary, $\lim_{\ell\rightarrow\infty }\ell\mathbf E[|R(\widehat \theta, \overline \theta)|\left| \bm N\right.] = 0$.  
Combining with \eqref{eqn:AICc/proof/part/one} and \eqref{eqn:AICc/proof/part/two}, this completes our proof. 
\end{proof}
\section{Proof of Theorem \ref{thm:randmat}}\label{appendix:FixedPointCriterion}
\begin{proof}[Proof of Theorem \ref{thm:randmat}]
We write $\overline P = \overline U\overline \Sigma \overline V^\top$ for a singular value decomposition of $\overline P$, 
and write $\widehat P = \widehat U\widehat \Sigma \widehat V^\top$ for a singular value decomposition of $\widehat P$. 
Since $\overline P$ and $\widehat P$ are symmetric, $\overline U = \overline V$ and $\widehat U = \widehat V$.
Let $\xi := \overline P(\widehat P)^\top (\overline U \overline \Sigma^{-2} \overline U^\top)$,
and note that by Condition \ref{cond:ASEmeetsNMF},
\begin{align*}
e_n(\widehat W;\overline P) :=  \widehat W - \xi  \widehat W  = \widehat Y - \xi  \widehat Y =  \widehat U - \xi  \widehat U.
\end{align*}
Note that $\varepsilon_{n}(\widehat W;\overline P) = \|e_{n}(\widehat W;\overline P)\|_{F}$. Now, 
\begin{align*}
\xi 
=
\overline U\overline \Sigma \overline V^\top  \widehat V  \widehat \Sigma \widehat U^\top \overline U \overline \Sigma^{-2} \overline U^\top 
 =
\overline U\overline \Sigma (\overline V^\top (\widehat V  - \overline V) + I) \widehat \Sigma  (\overline U^\top (\widehat U -\overline U) +I)^\top \overline \Sigma^{-2} \overline U^\top.
\end{align*}
Then,  
\begin{align}
e_n(\widehat W;\overline P) 
& =  \widehat{U} - \overline U \overline \Sigma  \widehat\Sigma  \overline \Sigma^{-2} \overline U^\top \widehat U
- \overline U\overline \Sigma (\overline V^\top (\widehat V  - \overline V) ) \widehat \Sigma  (\overline U^\top (\widehat U -\overline U))^\top \overline \Sigma^{-2} \widehat U^\top  \nonumber \\
&\quad\quad - \overline U\overline \Sigma (\overline V^\top (\widehat V  - \overline V) ) \widehat \Sigma  \overline \Sigma^{-2}  \widehat U^\top 
- \overline U\overline \Sigma \widehat \Sigma  (\overline U^\top (\widehat U -\overline U))^\top \overline \Sigma^{-2} \widehat U^\top. \label{eqn:thmproof-inequality-1}
\end{align}
Also, specifically for the first two terms, we have 
\begin{align}
&\quad \widehat{U} - \overline U \widehat \Sigma  \overline \Sigma^{-1} \overline U^\top  \widehat U \\
&= \widehat U - \overline U +\overline U - \overline U \overline U^\top \widehat U + \overline U \overline U^\top \widehat U - \overline U \widehat \Sigma \overline  \Sigma^{-1} \overline U^\top \widehat U \nonumber \\
&= (\widehat U - \overline U)  +  \overline U \overline U^\top (\overline U - \widehat U) + \overline U (I - \widehat \Sigma  \overline \Sigma^{-1}) \overline U^\top \widehat U \nonumber\\
&= (I  -  \overline U \overline U^\top) (\widehat U - \overline U) + \overline U (I - \widehat \Sigma  \overline \Sigma^{-1}) \overline U^\top \widehat U. \label{eqn:thmproof-inequality-2}
\end{align}
Hence,
\begin{align}
 \|e_n(\widehat W;  \overline P) \|_F  \le 
&\|I  -  \overline U \overline U^\top\|_F\|\widehat U - \overline U\|_F +  r^{3/2} \|I - \widehat \Sigma  \overline \Sigma^{-1}\|_F  \nonumber \\
&
+  r^{2} \|\overline \Sigma \|_F   \|\widehat V  - \overline V \|_F \|\widehat \Sigma\|_F   \|\widehat U^\top -\overline U^\top\|_F  \|\overline \Sigma^{-2}\|_F \nonumber  \\
&
+ r^{3/2} \|\overline \Sigma \|_F \|\widehat V  - \overline V \|_F \|\widehat \Sigma\|_F \|\overline \Sigma^{-2}\|_F  \nonumber  \\
&
+ r^{3/2} \|\overline \Sigma\|_F \|\widehat \Sigma\|_F  \|\widehat U^\top - \overline U^\top \|_F \| \overline \Sigma^{-2}\|_F. \nonumber 
\end{align}
First, by Proposition 4.5 and 
Theorem 4.6 in \cite{AAetal-2013}, for each $\varepsilon > 0$, 
for all sufficiently large values of $n$,  with probability $1-\varepsilon$, 
\begin{align}
&\|\overline U - \widehat{U} \|_F \le 4 \delta_r^{-2}  \sqrt{2r \log(n/\varepsilon)/n}, \label{eqn:AAbound2}
\end{align}
Appealing Proposition 4.5 of \cite{AAetal-2013} once again, we also have that
\begin{align}
&\|\widehat \Sigma\|_F \| \overline \Sigma^{-1}\|_F 
\le 
r (\|\widehat \Sigma\|/ \|  \overline \Sigma\|)   \| \overline  \Sigma\| \| \overline  \Sigma^{-1}\|  
\le 
r (2/\delta_r) \xi_0,\\
&\| \overline \Sigma\|_F\|\widehat \Sigma\|_F  \| \overline \Sigma^{-2}\|_F 
\le  r^{3/2}  
\| \overline \Sigma\|^2  \| \overline  \Sigma^{-2}\|  (\|\widehat \Sigma\|/\| \overline \Sigma\|)
\le r^{3/2} \xi_0^2 2/\delta_r, 
\end{align}
where $\|\cdot\|$ denotes the spectral norm, i.e., the largest singular value of
the matrix. 
Therefore,
\begin{align*}
\|e_n(\widehat W;  \overline P)\|_F & \le 
(\sqrt{n}+{r}) \|\widehat U-\overline U\|_F + r^{3/2}  (\sqrt{r} + \| \widehat \Sigma\|_F\| \overline  \Sigma^{-1}\|_F)   \\
& + r^2 \| \overline \Sigma\|_F \|\widehat \Sigma\|_F\| \overline \Sigma^{-2}\|_F \|\widehat U- \overline U\|_F^2 \\
& +  r^{3/2} \|\widehat U - \overline U\|_F \|\widehat \Sigma\|_F\| \overline \Sigma\|_F  \| \overline  \Sigma^{-2}\|_F\\
& + r^{3/2}   \|\widehat U- \overline U\|_F \| \overline \Sigma\|\|\widehat \Sigma\|_F \| \overline \Sigma^{-2}\|_F.
\end{align*}
Therefore, we have the following inequality from which our claim follows:
\begin{align*}
\|e_n(\widehat W; \overline P)\|_F & \le 
(\sqrt{n}+{r}) \|\widehat U-\overline U\|_F
+ r^{3/2}  (\sqrt{r} + r (2/\delta_r)\xi_0)   \\
& + r^2   r^{3/2} \xi_0^2 2/\delta_r \|\widehat U- \overline U\|_F^2 \\
& + 2 r^{3/2} r^{3/2} \xi_0^2 2/\delta_r  \|\widehat U - \overline U\|_F.
\end{align*}
Since $\varepsilon >0$ were arbitrarily chosen, our claim follows from this.
\end{proof}
\section{Algorithm Listings}\label{sec:algorithms}
In Algorithm \ref{algo:mainalgorithm}, the symbol $\ISVT(M;r)$ denotes performing singular value thresholding on the matrix $M$ assuming that its rank is $r$.
In Algorithm \ref{algo:mainalgorithm}, the symbol $\NMF(M;r)$ denotes performing non-negative matrix factorization on the matrix $M$ assuming that its inner dimension is $r$.    We mention that in all of our experiments, 
to protect against the effect of the initial seed used for the underlying NMF algorithm, we have conducted 
\emph{multiple} runs of our clustering-of-graph procedure, and choose $r$ with the minimum AICc value. 
For our numerical experiments, $\ISVT$ is implemented so that singular value thresholding is iteratively performed until 
the outputs from two consecutive runs differ only by a small threshold value in $\|\cdot \|_{F}$.   
  \begin{algorithm}
    \caption{Clustering of graphs}
    \label{algo:mainalgorithm}
\begin{algorithmic}[1]
    \Require $X$, $r$
    \vskip0.1in    
    \Procedure {gclust}{$X$,$r$}
    \State $\widehat X \leftarrow \ISVT(X;r)$
    \State $\widehat X \leftarrow \widehat X \diag(\bm 1^\top \widehat X)^{-1}$
    \State $(\widehat W,\widehat H) \leftarrow \NMF(\widehat X;r)$
    \State \Return $\widehat W$, $\widehat H$
    \EndProcedure
\end{algorithmic}
\end{algorithm}
\begin{algorithm}
    \caption{Choosing the number of clusters for clustering of graphs}
    \label{algo:mainalgorithm2}
\begin{algorithmic}[1]
    \Require $X$
    \vskip0.1in    
    \Procedure{getGclustModelDim}{X}
    \For{$r \gets 1,\ldots, T$}
	\State $(\widehat W,\widehat H) \gets \textsc{gclust}(X,r)$
	\State $f(r)  \gets \textsc{AICc}(\widehat W, \widehat H, X,r)$
    \EndFor
    \State $\widehat r \gets \argmin_{r=1}^{T} f(r)$
    \State \Return $\widehat r$
    \EndProcedure
\end{algorithmic}
\end{algorithm}
\section{Additional Numerical Examples}\label{sec:more.num.expr}
\paragraph{Data with a ground truth}
As far as we know, there is no similar work that is directly comparable to ours.
As such, in our next examples, we apply our technique to some numerical examples that have
been considered for finding the inner dimension of non-negative matrix factorization on a matrix derived from images
for computer vision application.  In Example \ref{[exa:swimmer+visual.tfidf]}, the correct inner dimension is $16$,
and in Example \ref{[exa:mit-indoor-scene]}, the correct inner dimension is $3$. 
\begin{example}\label{[exa:swimmer+visual.tfidf]}
The swimmer data set is a frequently-tested data set for bench-marking 
NMF algorithms (c.f.~\cite{Donoho03whendoes} and \cite{NickGillis2014}). 
In our present notation, each column of $220\times 256$ 
data matrix $X$ 
is a vectorization of a binary image, and each row corresponds to 
a particular pixel.  
Each image is a binary images ($20$-by-$11$ pixels) 
of a body with four limbs which can be each in four different 
positions.  Technically speaking, the matrix $X$ is $16$-separable 
while the rank of $X$ is $13$.   This amounts to saying that $X$ represents a time-series 
of $16$ recurring motifs, and the rank of $X$ being $13$ is a nuisance fact.
Note that Condition \ref{cond:rank=innerdim} is violated.  
Nevertheless, application of our AICc criteria yields the estimated $\widehat r$ as $16$.
We mention that to protect against the effect of an initial seed used for the 
underlying NMF algorithm, we have used multiple runs of our clustering-of-graphs 
procedure, and choose $\widehat r$ with the smallest AICc value.   The AICc values are 
reported in Table \ref{tab:aic-swimmer}.
\hfill$\Box$
\end{example}
\begin{table}
\centering
\caption{AICc values for Example \ref{[exa:swimmer+visual.tfidf]}. The fact that there are $16$ image types 
coincides with the fact that the AICc value is minimized at $16$. }\label{tab:aic-swimmer}
\begin{tabular}{cccc}
$\widehat r$ & Loss &  Penalty &      AICc \\ \hline
   12 &     947.0524&  0.346153846& 947.3986\\
   13 &     901.5910&  0.483201589& 902.0742\\
   14 &     895.7349&  0.565097295& 896.3000\\
   15 &     865.9876&  0.748465296& 866.7361\\
   16 &     834.6471&  0.939686092& {\bf 835.5867}\\ 
   17 &     865.9512&  7.074993387& 873.0262\\ \hline
\end{tabular}
\end{table}
\begin{example}\label{[exa:mit-indoor-scene]}
We consider a $200\times 1425$ data matrix  $X$, each of whose columns is associated with 
an image and each of whose rows represents a visual feature.  Each column of $X$ is a representation of its 
associated image by way of a ``bag of visual words'' approach.  
Specifically, first, from each image,  one extracts a bag of SIFT-features, and then uses $K$-means clustering of 
a collection of bags of SIFT-features to obtain dimensionality reduction, yielding $200$ visual features.   
Each image corresponding to a column of $X$ can be attributed to $3$ types, ``bowling'', ``airport'', and ``bar''. 
Our AICc procedure yields that the AICc value is minimized at the inner dimension $\widehat r=3$. The 
AICc values are reported in Table \ref{tab:mit-aic}.
\hfill$\Box$
\end{example}
\begin{table}
\centering
\caption{AICc values for Example \ref{[exa:mit-indoor-scene]}. The fact that there are $3$ image types 
coincide with the fact that the AICc value is minimized at $3$. }\label{tab:mit-aic}
\begin{tabular}{cccc}
$\widehat r$ & Loss &  Penalty &      AICc \\ \hline
1&      6882.495& 0.003716505& 6882.499\\ 
2&      6788.166& 0.015255502& 6788.182\\
3&      6681.398& 0.034413464& {\bf 6681.432}\\ 
4&      6814.334& 0.073448070& 6814.407\\
5&      6792.356& 0.121808012& 6792.477\\
6&      6749.916& 0.157356882& 6750.073\\ \hline
\end{tabular}
\end{table}
\bibliographystyle{apalike}

\end{document}